\documentclass[11pt,a4paper]{article}

\usepackage{theorem}
\usepackage{pifont}
\usepackage{enumitem}
\usepackage{amsmath, amssymb,amsfonts} 
\usepackage{xspace}
\usepackage{microtype}
\usepackage{misc}
\usepackage{enumitem}
\usepackage{color}
\usepackage{algorithm}
\usepackage{algpseudocode}
\usepackage{tikz}
\definecolor{Black}  {RGB}{0,0,0}

\pgfdeclarelayer{edgelayer}
\pgfdeclarelayer{nodelayer}
\pgfsetlayers{edgelayer,nodelayer,main}

\tikzstyle{none}=[inner sep=0pt]

\tikzstyle{rn}=[circle,fill=Red,draw=Black,line width=0.8 pt]
\tikzstyle{gn}=[circle,fill=White,draw=Black,line width=0.8 pt]
\tikzstyle{yn}=[circle,fill=Yellow,draw=Black,line width=0.8 pt]
\tikzstyle{simple}=[circle,fill=White,draw=Black]
\tikzstyle{newstyle1}=[circle,fill=Black,draw=Black,line width=0.3 pt,inner sep=0pt]
\tikzstyle{simple2}=[-,dashed,draw=Black]
\tikzstyle{simpledotted}=[-,dotted,draw=Black]
\tikzstyle{simple}=[-,draw=Black,line width=2.000]
\tikzstyle{arrow}=[-,draw=Black,postaction={decorate},decoration={markings,mark=at position .5 with {\arrow{>}}},line width=2.000]
\tikzstyle{tick}=[-,draw=Black,postaction={decorate},decoration={markings,mark=at position .5 with {\draw (0,-0.1) -- (0,0.1);}},line width=2.000]
\tikzstyle{newstyle2}=[-latex,draw=Black]
\tikzstyle{newstyle3}=[->,dotted,draw=Black]
\tikzstyle{newstyle6}=[->,dotted,draw=Black]

\def\mvd{\ensuremath{\rightarrow\!\!\!\rightarrow}}  
\newcommand{\f}{\ensuremath{f_{\textsf{\tiny MEM}}}}
\newcommand{\g}{\ensuremath{f_{\textsf{\tiny EQ}}}}

\newenvironment{remark}[1]{\par\noindent\underline{Remark:}\space#1}{}

\newenvironment{proof}[1]{\textit{Proof.}\space#1}{\unskip\nobreak\hfil\penalty50
   \hskip2em\hbox{}\nobreak\hfil
   \ding{111}%
   \parfillskip=0pt \finalhyphendemerits=0
    \medskip\goodbreak\noindent\ignorespacesafterend}
    \theoremstyle{plain}
\newtheorem{theoremaux}{Theorem}
\newenvironment{theorem}{\begin{theoremaux}}{\end{theoremaux}\parfillskip=0pt\finalhyphendemerits=0\goodbreak\noindent\ignorespacesafterend}
\newtheorem{lemmaaux}[theoremaux]{Lemma}
\newenvironment{lemma}{\begin{lemmaaux}}{\end{lemmaaux}\parfillskip=0pt\finalhyphendemerits=0\goodbreak\noindent\ignorespacesafterend}
\newtheorem{corollaryaux}[theoremaux]{Corollary}
\newenvironment{corollary}{\begin{corollaryaux}}{\end{corollaryaux}\parfillskip=0pt\finalhyphendemerits=0\goodbreak\noindent\ignorespacesafterend}
\theoremstyle{definition}
\newtheorem{definitionaux}[theoremaux]{Definition}
\newenvironment{definition}{\begin{definitionaux}}{\end{definitionaux}\parfillskip=0pt\finalhyphendemerits=0\goodbreak\noindent\ignorespacesafterend}

\usepackage{times}
 
 \usepackage{authblk}

\begin{document}

\title{New Steps on the Exact Learning of CNF}

\author[1]{Montserrat Hermo}
\author[2]{Ana Ozaki}
\affil[1]{Languages and Information Systems, Univ. of the Basque Country, Spain}
\affil[2]{Department of Computer Science, Univ. of Bras\'{i}lia, Brazil}

\maketitle
\begin{abstract}
A major problem in computational learning theory  
is whether the class of formulas in conjunctive normal form (CNF) is efficiently learnable.
Although it is known that this class cannot be polynomially learned using either
membership or equivalence queries alone, it is open whether CNF can be
polynomially learned using both types of queries. 
One of the most important results concerning a restriction of the class CNF
is that propositional Horn formulas are polynomial time learnable 
 in Angluin's exact
learning model with membership and equivalence queries.   
In this work we push this boundary and show that the class of 
multivalued dependency formulas (MVDF) is polynomially learnable from interpretations.
 We then provide a
notion of reduction between learning problems in Angluin's model, showing that a transformation
of the algorithm suffices to efficiently learn multivalued database dependencies from data relations. We also
show via reductions that our main result extends well known previous results and allows us to find
alternative solutions for them. 
\end{abstract}


\section{Introduction}

In the exact learning model, proposed by Angluin~\cite{DBLP:journals/iandc/Angluin87}, 
a learner tries to identify an abstract target set by posing queries 
to an oracle. The most successful protocol
  uses  membership and equivalence queries~\cite{kearns1994introduction}. 
The exact learning model is distinguished by many other machine learning techniques 
for being a purely deductive reasoning approach. 
 Since its proposal,  
a number of researchers have investigated which concept classes can be 
polynomially learned  and 
it is  known that    
algorithms in this model 
can be transformed into solutions    
for other well known settings such as the PAC~\cite{valiant1984theory,Angluin1988}
and the online machine learning~\cite{Littlestone1988} models extended with membership queries. 

Restrictions of CNF and DNF which have been proved to be polynomially learnable 
with membership and equivalence queries include: monotone DNF (DNF formulas with no
negated variables)~\cite{Angluin1988}; $k$-clause CNF 
(CNF  formulas with at most   $k$ clauses)~\cite{Angluin1987} and 
 read-twice DNF 
(DNF where each variable occurs at most twice)~\cite{PillaipakkamnattRaghavan1995}.
The
CDNF class (boolean functions whose  CNF size is polynomial in its DNF size)~\cite{Bshouty1995} 
is also known to be learnable in polynomial time with both types of queries. 
Despite the intense effort to establish the complexity of 
learning the full classes of CNF and DNF, 
the complexity of these classes in  the exact learning model with both queries 
     remains open. 
It is known that these classes cannot be polynomially learned
using either membership or equivalence queries
alone~\cite{Angluin1988,DBLP:journals/ml/Angluin90} and  some advances in 
proving hardness of DNF with both queries appears in~\cite{Hellerstein:2002:ELD:509907.509976}.

A classical result concerning a restriction of the class 
CNF appears in~\cite{DBLP:journals/ml/AngluinFP92}, 
where propositional Horn formulas 
are proved to be polynomially learnable with membership and equivalence queries. 
In fact, Horn is a special case of a class called $k$-quasi-Horn:  
 clauses with at most $k$ unnegated literals. However, it is pointed out 
 by Angluin et. al~\cite{DBLP:journals/ml/AngluinFP92} that, even for $k = 2$, 
 learning the class of $k$-quasi-Horn formulas is as hard as learning CNF (Corollary 25 of~\cite{frazier1994matters}).
Thus, if exact learning CNF  is indeed intractable, 
the boundary of what can be learned in polynomial time with queries lies between 
$1$-quasi-Horn (or simply Horn) and $2$-quasi-Horn formulas. 
In this work we study the class of multivalued dependency formulas (MVDF)~\cite{Sagiv1981},  
 which (as we explain in the Preliminaries) 
is a 
natural restriction of $2$-quasi-Horn  and 
a non-trivial generalization of Horn.

Another motivation to study the complexity of learning MVDF is that 
this class is the logical theory behind 
multivalued   dependencies (MVD)~\cite{Sagiv1981,Balcazar2004}, in the sense that 
one can map a set of  multivalued
dependencies to a multivalued  dependency formula 
 preserving the logical consequence relation. 
A similar  equivalence between functional 
dependencies and propositional Horn formulas is given by the authors of~\cite{fagin1977functional}.  
Although data dependencies are usually determined from the semantic
attributes, they may not be known a priori by database designers.
 Discovering functional and multivalued 
 dependencies from examples of data relations using inductive 
  reasoning has been investigated 
 by~\cite{kantola1992discovering,mannila1994algorithms,huhtala1998efficient,Flach:1999:DDD:1216155.1216159}. 
Here we study this problem in Angluin's   model. In this paper, we give a
polynomial time algorithm that exactly learns multivalued dependencies formulas (MVDF) from 
interpretations. We then provide  
a formal notion of reduction for the exact learning model 
and use this notion to reduce the problem of learning MVD 
from data relations (and other problems below) to the problem of 
learning MVDF from interpretations.

\vspace{0.2cm}

{\bf Previous results.} A large part of the related work was already mentioned. We now  
  discuss some previous results which are extended by the present work.  
A  polynomial time algorithm for exact  learning (with membership and equivalence queries) 
propositional Horn from interpretations 
 was first presented by Angluin et. al~\cite{DBLP:journals/ml/AngluinFP92} (also, see~\cite{Arias2011}). 
One year later, Frazier and Pitt presented 
a polynomial time algorithm for exact  learning 
propositional Horn from entailments~\cite{DBLP:conf/icml/FrazierP93}. 
More recently, Lav\'{i}n proved polynomial time exact learnability of 
CRFMVF (resp., CRFMVD), which is 
a restriction 
of MVDF (resp., MVD)~\cite{Puente15}. 
Then, 
a polynomial time algorithm for exact learning the full class MVDF 
from $2$-quasi-Horn clauses was presented by the authors of~\cite{hermo2015exact}.

 \begin{figure}[h]
\begin{center}
\begin{tikzpicture}
	\begin{pgfonlayer}{nodelayer}
		\node [style=none] (0) at (-5.25, 4) {};
		\node [style=none] (1) at (-3, 3.5) {};
		\node [style=none] (2) at (-8.25, 3.5) {};
		\node [style=none] (3) at (-6, 3.5) {};
		\node [style=none] (4) at (-4.75, 3.5) {};
		\node [style=none] (5) at (-5.5, 4.25) {MVDF$_\mathcal{I}$};
		\node [style=none] (6) at (-5.75, 4) {};
		\node [style=none] (7) at (-4.75, 4) {};
		\node [style=none] (8) at (-6.25, 4) {};
		\node [style=none] (9) at (-8.75, 3.25) {HORN$_{\mathcal{I}}$};
		\node [style=none] (10) at (-2.5, 3.25) {CRFMVF$_{\mathcal{I}}$};
		\node [style=none] (11) at (-8.5, 3) {};
		\node [style=none] (12) at (-2.75, 3) {};
		\node [style=none] (13) at (-8.75, 2.5) {};
		\node [style=none] (14) at (-2.5, 2.5) {};
		\node [style=none] (15) at (-4.25, 3.25) {MVD$_{\mathcal{R}}$};
		\node [style=none] (16) at (-5.75, 3.25) {MVDF$_{\mathcal{Q}}$};
		\node [style=none] (17) at (-8.75, 2.25) {HORN$_{\mathcal{E}}$};
		\node [style=none] (18) at (-2.25, 2.25) {CRFMVD$_{\mathcal{R}}$};
		\node [style=none] (19) at (-7.25, 3) {MVDF$_{\mathcal{E}}$};
		\node [style=none] (20) at (-6, 4) {};
		\node [style=none] (21) at (-7, 3.25) {};
		\node [style=none] (22) at (-8.25, 2.5) {};
		\node [style=none] (23) at (-7.5, 2.75) {};
		\node [style=none] (24) at (-6.25, 3) {};
		\node [style=none] (25) at (-7, 2.75) {};
		\node [style=none] (26) at (-7, 3.5) {?};
		\node [style=none] (27) at (-7, 2.5) {};
		\node [style=none] (28) at (-6, 3) {};
		\node [style=none] (29) at (-6.25, 2.5) {?};
	\end{pgfonlayer}
	\begin{pgfonlayer}{edgelayer}
		\draw [style=newstyle2] (4.center) to (0.center);
		\draw [style=newstyle2] (3.center) to (6.center);
		\draw [style=newstyle2] (2.center) to (8.center);
		\draw [style=newstyle2] (13.center) to (11.center);
		\draw [style=newstyle2] (14.center) to (12.center);
		\draw [style=newstyle2] (12.center) to (14.center);
		\draw [style=newstyle6] (21.center) to (20.center);
		\draw [style=newstyle2] (22.center) to (23.center);
		\draw [style=newstyle2, bend left, looseness=1.00] (24.center) to (25.center);
		\draw [style=newstyle2] (0.center) to (4.center);
		\draw [style=newstyle6, bend right=45, looseness=0.75] (27.center) to (28.center);
		\draw [style=newstyle2] (1.center) to (7.center);
	\end{pgfonlayer}
\end{tikzpicture}
\caption{Reductions among learning problems}\label{fig:reductions}
\end{center}
\end{figure}
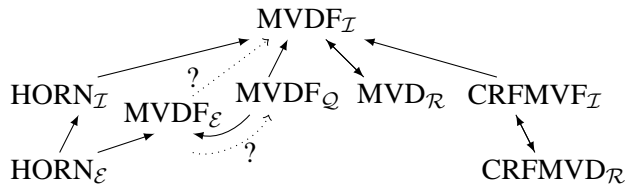

Figure~\ref{fig:reductions} shows the relationship among learning 
problems via reductions, where $C_{E} \rightarrow C'_{E'}$ means that:  
the problem of exactly learning (with membership and equivalence queries) the 
class $C$ from the examples $E$ is reducible in polynomial time to the problem of 
exactly learning the class $C'$ from $E'$. We use \Imc for interpretations, $\Emc$ for 
entailments, $\Qmc$ for $2$-quasi-Horn clauses and $\Rmc$ for data relations.  
As shown in Figure~\ref{fig:reductions}, the problem MVDF$_\Imc$, solved in 
the present work, extends previous results on the efficient learnability of 
data dependencies and their corresponding propositional formulas. 
Our positive result for  MVDF$_{\Imc}$ is a non-trivial extension of HORN$_\Imc$ (in~\cite{DBLP:journals/ml/AngluinFP92})
and CRFMVF$_{\Imc}$ (in~\cite{Puente15}) and 
allow us to prove for the first time the polynomial time 
learnability of the full class of multivalued dependencies from data relations (MVD$_{\Rmc}$). 
As shown in Figure~\ref{fig:reductions}, one can reduce HORN$_{\Emc}$ to HORN$_{\Imc}$. However, 
we did not find a way of reducing MVDF$_\Emc$ to MVDF$_\Imc$ and we leave open  
the  question of whether MVDF$_\Emc$ is polynomial time 
exactly learnable. 

\section{Preliminaries}

\paragraph{\upshape\textbf{Exact Learning}} Let $E$ be a set of examples (also called \emph{domain} or \emph{instance space}). 
A \emph{concept over} $E$ is a subset of $E$ and a \emph{concept class} is a set $C$  of  concepts over $E$.  
Each concept $c$ over $E$ induces a dichotomy of 
\emph{positive} and \emph{negative} examples, meaning that $e \in c$ 
is a positive example and $e \in E \setminus c$ is a negative example. 
For computational purposes, concepts need to be specified by some representation. 
So we define a \emph{learning framework}  to be a triple $(E,\Lmc,\mu)$, where 
$E$ is a set of examples,   
 $\Lmc$ is a set of \emph{concept representations}  
and $\mu$ is a surjective function from $\Lmc$ 
to a concept class $C$ of concepts over $E$.

Given a learning framework $\Fmf = (E,\Lmc,\mu)$, for each $l \in \Lmc$, 
denote by ${\sf MEM}_{l,\Fmf}$ the oracle that takes as input some $e \in E$ and 
returns `yes' if $e \in \mu(l)$ and `no' otherwise. 
A \emph{membership query} is a call to an oracle ${\sf MEM}_{l,\Fmf}$ with 
some $e \in E$ as input, for  $l \in \Lmc$ and $E$. 
Similarly, for every $l \in \Lmc$,
we denote by ${\sf EQ}_{l,\Fmf}$ the oracle that takes as input 
a concept representation $h \in \Lmc$ and returns `yes', if $\mu(h) = \mu(l)$, 
or a \emph{counterexample} $e \in \mu(h) \oplus \mu(l)$, otherwise.
An \emph{equivalence query} is a call to an oracle ${\sf EQ}_{l,\Fmf}$ with 
some $h \in \Lmc$ as input, for  $l \in \Lmc$ and $E$.  
We say that a learning framework $(E,\Lmc,\mu)$ is \emph{exactly learnable} if
there is an algorithm $A$ such that for any {target} $l\in \Lmc$ the algorithm $A$
always halts and outputs $l'\in\Lmc$ such that $\mu(l) = \mu(l')$ using 
membership and equivalence queries answered by the oracles ${\sf MEM}_{l,\Fmf}$ and
${\sf EQ}_{l,\Fmf}$, respectively.
A learning framework $(E,\Lmc,\mu)$ is
\emph{polynomial time} exactly learnable if it is exactly learnable by a deterministic algorithm $A$ 
such that at every step of computation the time used by $A$ up to that step is bounded by
a polynomial $p(|l|,|e|)$, where $l$ is the target and $e \in E$ is the largest
counterexample seen so far\footnote{We count each call to an oracle as one step of
computation. Also, we assume some natural notion of  length for an
example $e$ and a concept representation $l$, denoted by $|e|$ and $|l|$, respectively.}.
 
\paragraph{\upshape\textbf{    Multivalued Dependency Formulas}}
Let $V$ be a finite set of symbols, representing boolean variables. 
The logical constant \emph{true} is represented by $\mathbf{T}$ and the 
logical constant \emph{false} is represented by $\mathbf{F}$.
 A \emph{multivalued} (for short \emph{mvd}) \emph{clause}  is an implication $X \rightarrow Y \vee Z$, where 
 $X$,  $Y$ and $Z$ are pairwise disjoint conjunctions of variables from $V$ 
 and $X \cup Y \cup Z = V$. We note that some of $X,Y,Z$ may be empty. 
 An \emph{mvd formula} is a conjunction of mvd clauses.  
 A \emph{$k$-quasi-Horn clause} is a propositional clause containing at most $k$ unnegated literals. 
A \emph{$k$-quasi-Horn formula} is a conjunction of $k$-quasi-Horn clauses.   
A \emph{Horn clause} (resp., \emph{Horn formula}) is a   $k$-quasi-Horn clause 
(resp., \emph{$k$-quasi-Horn formula}) with $k=1$. 

\begin{remark}\label{rem:defs}\upshape
From the definition of an mvd clause  and a k-quasi-Horn clause  it is easy to see that:
\begin{enumerate} 
\item 
 any Horn clause is logically equivalent to a set of 2 mvd clauses. 
For instance, the Horn clause  
$135 \rightarrow 4$,
 is equivalent to: 
$\{12356 \rightarrow 4, 135 \rightarrow 4 \vee 26\}$; 
\item 
 any mvd clause is logically equivalent to a conjunction 
of 2-quasi-Horn clauses with size polynomial in the number of variables. For instance, the mvd clause 
$1 \rightarrow 23 \vee 456$,
by distribution, is equivalent to: 
$\{1 \rightarrow 2 \vee 4, 1 \rightarrow 2 \vee 5, 
1 \rightarrow 2 \vee 6, 1 \rightarrow 3 \vee 4, 1 \rightarrow 3 \vee 5, 1 \rightarrow 3 \vee 6\}$.
\end{enumerate}
 \end{remark}

 To simplify the notation, we treat sometimes conjunctions as sets and vice versa. 
Also, if for example $V = \{v_1, v_2, v_3, v_4, v_5, v_6\}$ is a set of variables and 
$\varphi =  (v_1 \rightarrow (v_2 \wedge v_3 ) \vee (v_4 \wedge v_5 \wedge v_6 )) 
\wedge ((v_2 \wedge v_3) \rightarrow  (v_1 \wedge v_5 \wedge v_6 ) \vee v_4 )$ 
is a formula then  
 we write $\varphi$ in this shorter way: 
 $\{1 \rightarrow 2 3 \vee 4 5 6, 2 3 \rightarrow 1 5 6 \vee 4\}$, 
 where conjunctions between variables are omitted and each propositional 
 variable $v_i \in V$ is mapped to $i \in \mathbb{N}$. 
 For the purposes of this paper, 
  we treat $X \rightarrow Y \vee Z$ and $X \rightarrow Z \vee Y$ as 
 distinct mvd clauses, where $Y$ and $Z$ are non-empty. For example, $12\rightarrow 34 \vee 56$ and $12\rightarrow 56 \vee 34$ 
 are counted as two  distinct mvd clauses.

An interpretation \Imc is a mapping from $V \cup \{\mathbf{T}, \mathbf{F}\}$ to $\{true,false\}$, where 
$\Imc (\mathbf{T}) = true$ and $\Imc (\mathbf{F}) = false$.
We denote by ${\sf true}(\Imc)$  the set of variables assigned to \emph{true} 
in $\Imc$. In the same way, let ${\sf false}(\Imc)$ be the set of variables assigned to \emph{false}  
in $\Imc$. Observe that ${\sf false}(\Imc) = V \setminus {\sf true}(\Imc)$. 
We follow the terminology provided by \cite{DBLP:journals/ml/AngluinFP92} and  say that 
an interpretation \Imc \emph{covers}  $X \rightarrow Y \vee Z $ if $X \subseteq {\sf true}(\Imc)$. 
An interpretation \Imc \emph{violates}  $X \rightarrow Y \vee Z $ if  
\Imc \emph{covers}  $X \rightarrow Y \vee Z $ 
and: (a) $Y$ and $Z$ are non-empty and there are $v \in Y$ and $w \in Z$ such that $v,w \in {\sf false}(\Imc)$;
or (b) one of $Y,Z$ is empty and there is $v \in Y\cup Z$ such that ${\sf false}(\Imc)= \{v\}$; or 
(c) ${\sf false}(\Imc)= \emptyset$ and $X \rightarrow Y \vee Z$ is the mvd clause $V \rightarrow \mathbf{F}$. 
If \Imc does not violate $X \rightarrow Y \vee Z $ then we write $\Imc\models X \rightarrow Y \vee Z $. 

Given two interpretations \Imc and $\Imc'$, we define $\Imc\cap\Imc'$ to be the interpretation  
such that ${\sf true}(\Imc\cap\Imc') = {\sf true}(\Imc) \cap {\sf true}(\Imc')$. 
If $\Smc$ is a sequence of interpretations and \Imc is an interpretation 
 occurring at position $i$ then 
we write $\Imc_i \in \Smc$. 
Also, we denote by $\Smc\cdot \Imc$ the result of appending \Imc to $\Smc$.
The learning MVDF from interpretations framework is  defined as 
 $\Fmf\text{(MVDF$_\Imc$)}=(E_\Imc,\Lmc_{\sf M},\mu_\Imc)$, where $E_\Imc$ 
is the set of all interpretations in the propositional variables $V$ under consideration, 
$\Lmc_{\sf M}$ is the set of all sets of mvd clauses that can be expressed in $V$ and, for every $\Tmc \in \Lmc_{\sf M}$, 
$\mu_\Imc(\Tmc) = \{\Imc \in E_\Imc\mid \Imc \models \Tmc\}$.

\section{Learning MVDF from Interpretations}
\label{sec:mvdf-interpretation}

In this section we present an  algorithm that learns the class MVDF from interpretations.   
More precisely, we show that the learning framework 
$\Fmf\text{(MVDF$_\Imc$)}$  
is polynomial time exactly learnable.

The learning algorithm for $\Fmf\text{(MVDF$_\Imc$)}$ is given by 
Algorithm \ref{alg:learner}. 
Algorithm \ref{alg:learner} 
 maintains a sequence \Pmf of interpretations which are 
positive examples for the target \Tmc and a sequence $\Lmf$ 
of interpretations which are negative examples (for the target \Tmc).
The learner's hypothesis \Hmc is constructed using both \Pmf and \Lmf. 
In order to learn all of  the mvd clauses in \Tmc, 
 we would like that mvd clauses induced by the 
 elements of \Pmf and \Lmf  approximate distinct mvd clauses in \Tmc. 
 This will happen if  at most polynomially many  elements in \Lmf violate
the same mvd clause in \Tmc. 
Overzealous refinement of a sequence of interpretations is 
a situation described by \cite{DBLP:journals/ml/AngluinFP92}. It may 
result in a loop   where we have several elements of the sequence 
violating  
the same clause in the target. 
We avoid this in Algorithm \ref{alg:learner} by (1) 
refining   negative counterexamples with elements of \Lmf (Line~\ref{line:cm})
and (2)  refining at most one (the first) element of  \Lmf per iteration (Line~\ref{line:replace}). 
We use the following notion,   provided by \cite{hermo2015exact},    
to
describe under which conditions the learner should 
refine either a negative counterexample or
an element of \Lmf. 

\begin{definition}
\label{def:good}
A pair $(\Imc, \Imc')$ of interpretations is a ${\sf good}{\sf Candidate}$ if: 
(i) ${\sf true}(\Imc \cap \Imc') \subset {\sf true}(\Imc)$; (ii) $\Imc \cap \Imc' \models \Hmc$; and 
(iii) $\Imc \cap \Imc' \not\models \Tmc$. 
 \end{definition}

 In the following we provide the main ideas of our proof (omitted proofs are given in full detail in the appendix). 
 If Algorithm~\ref{alg:learner} terminates, then it obviously has found a hypothesis \Hmc that is logically
equivalent to \Tmc, formulated with variables in $V$. It thus remains to show that Algorithm \ref{alg:learner} terminates in polynomial
time. 
In each iteration, one of the following three cases happens: 
\begin{enumerate}
\item a positive counterexample is 
added to the sequence \Pmf (Line~\ref{line:append-positive}); or 
\item  a negative example in \Lmf is replaced (Line~\ref{line:replace}); or 
\item a negative counterexample is appended to the sequence \Lmf (Line~\ref{line:append}). 
\end{enumerate}

\begin{algorithm} [h]
\begin{algorithmic}[1]
\State Let \Lmf be a sequence of negative examples and \Pmf a sequence of positive examples
\State
  Set $\Hmc_0 :=\{V\rightarrow \mathbf{F}\mid \Tmc \models V\rightarrow \mathbf{F}\}\cup 
 \{V\setminus\{v\}\rightarrow v\mid v \in V\text{ and }\Tmc \models V\setminus\{v\}\rightarrow v\} $   \label{line:hypo0}
\State Set $\Lmf :=\emptyset$, $\Pmf :=\emptyset$ and  $\Hmc := \Hmc_0$  \label{line:hypo} 
 
\While {$ \Hmc\not\equiv \Tmc$} 
\State Let \Imc be a counterexample \label{alg:pce}

\If {$\Imc\not\models \Hmc$} \label{line:decide}
\State Append \Imc to \Pmf \label{line:append-positive}
\Else  
\State Set $\Jmc:=$ \Call{RefineCounterexample}{$\Imc$,\Lmf} \label{line:cm}
\If {there is $\Imc_k \in \Lmf$ such that ${\sf goodCandidate}(\Imc_k, \Jmc)$} \label{line:cond}
\State Let $\Imc_i$ be the first in $\Lmf$ such that  ${\sf goodCandidate}(\Imc_i, \Jmc)$ \label{line:goodcandidate1}
\State Set $\Pmf':=$  \Call{UpdatePositiveExamples}{$  \Jmc,\Pmf,\Lmf$} and $\Pmf:=\Pmf'$ \label{line:updateseq}
\State Replace $\Imc_i \in \Lmf$ by $\Jmc$ \label{line:replace} 
\State Remove all $\Imc_j \in \Lmf\setminus\{\Jmc\}$ 
such that $\Imc_j \not\models$\Call{BuildClauses}{$\Jmc, \Pmf$} \label{line:remove-set}
\Else 
\State Append $\Jmc$ to $\Lmf$ \label{line:append}
\EndIf
\EndIf
\State Construct $\Hmc := \Hmc_0 \cup  \bigcup_{\Imc \in \Lmf}$\Call{BuildClauses}{$\Imc,\Pmf$} \label{line:transform}
\EndWhile
\end{algorithmic}
\caption{Learning algorithm for MVDF from Interpretations  \label{alg:learner}}
\end{algorithm}

\begin{algorithm} [h]
\begin{algorithmic}
\State Set $\Jmc:=\Imc$
\If {there is $\Imc_k \in \Lmf$ such that ${\sf goodCandidate}(\Imc, \Imc_k)$} \label{line:cond2}
\State Let $\Imc_i$ be the first in $\Lmf$ such that  ${\sf goodCandidate}(\Imc, \Imc_i)$ \label{line:goodcandidate2}
\State $\Jmc:=$\Call{RefineCounterexample}{$\Imc \cap \Imc_i$,\Lmf} \label{line:replace2}
\EndIf
\State \textbf{return}$(\Jmc)$
\end{algorithmic}
\caption{\textbf{Function} RefineCounterexample ($\Imc,\Lmf$) \label{alg:learning-mvdf-up}}
\end{algorithm}

To prove polynomial time learnability, we need to ensure that 
each iteration is done in polynomial time in the size of \Tmc and 
that the total number of iterations is also bounded. 
That is, the number of times  Cases 1, 2 and 3 happen is polynomial in the size of \Tmc. 
For Case 2 we note that 
each time a negative example is replaced, the number of variables 
assigned to true strictly decreases (Point (i) of Definition~\ref{def:good}). Then, 
Algorithm \ref{alg:learner} replaces each element of \Lmf at most $|V|$ times. 

Before we give a bound for  
Cases 1 and 3, we explain the bound on the number of recursive calls. 
We first note that in each recursive call 
of Function `RefineCounterexample'  (Algorithm \ref{alg:learning-mvdf-up})
the number of   variables 
assigned to true in a negative 
counterexample strictly decreases (Point (i) of Definition~\ref{def:good}).
This means that in each iteration of Algorithm \ref{alg:learner} the 
number of recursive calls of Function `RefineCounterexample' is 
at most $|V|$.  
To see the bound on the number of recursive calls of 
Function `UpdatePositiveExamples' (Algorithm~\ref{alg:remove}) we  use Lemma~\ref{lem:positive1}.   
By construction of $\Hmc_0$ (Line \ref{line:hypo0} of Algorithm \ref{alg:learner}) 
we can assume that all negative examples we deal with 
violate $X\rightarrow Y \vee Z \in \Tmc$ with $Y,Z$ non-empty\footnote{We note that   
one can easily check whether `$\Tmc\models V\rightarrow \mathbf{F}$' and 
`$\Tmc\models V\setminus\{v\}\rightarrow v$'   with 
 membership queries that receive interpretations as input.}. 
We write \Call{BuildClauses}{$\Imc, \Pmf$} for the set of mvd clauses returned as 
 output of Function `BuildClauses' 
(Algorithm~\ref{alg:build}) with \Imc and \Pmf as input.

\begin{lemma}
\label{lem:positive1}
Let $\Imc$ be a negative example. 
Let  \Call{BuildClauses}{$\Imc, \Pmf$}
$=\{ 
 {\sf true}(\Imc) \rightarrow Y_1 \vee Z_1,$ $ \ldots , {\sf true}(\Imc) \rightarrow Y_k \vee Z_k\} $.  
Then, for all $i, j$, such that $1 \leq i < j \leq k$, 
we have that:
$Y_i \cap Y_j = \emptyset \text{ and }  
\bigcup_{j=1}^k Y_j  = {\sf false}(\Imc).$ 
Moreover, for any ${\sf true}(\Imc)\rightarrow Y_i\vee Z_i$, $1\leq i \leq k$, 
we have that $Y_i$, $Z_i$ are non-empty.  
\end{lemma}

\begin{algorithm}[h]
\begin{algorithmic}
\State Set $X := {\sf true}(\Imc)$ and $C := \{ X \rightarrow v \vee V\setminus (X\cup \{v\})  \mid  v \in V \setminus X \}$ \label{line:initial}
\For {each $\Imc_l \in \Pmf$}  \label{line:firstFor}
\State Let $ X \rightarrow Y_1 \vee Z_1 ,   \ldots , X \rightarrow Y_k \vee Z_k  $ 
 be the mvd clauses in $C$ violated by $\Imc_l$
                    \label{line:Hmc1}
\State Replace in $C$ all these mvd clauses by 
$X \rightarrow \bigcup_{j=1}^k Y_j \vee (\bigcup_{j=1}^k Z_j\setminus \bigcup_{j=1}^k Y_j) $  \label{line:replace-function}     
\EndFor
\State \textbf{return} ($C$)	
\end{algorithmic}
\caption{\textbf{Function} BuildClauses ($\Imc, \Pmf$) 
\label{alg:build}}
\end{algorithm}
 
\begin{algorithm}[h]
\begin{algorithmic}
\State Set $\Pmf':=\Pmf$ 
\If {there are distinct $\Imc_k,\Imc_l \in \Lmf$ such that 
$\Imc_k\cap\Imc_l\not\models$\Call{BuildClauses}{$ \Kmc, \Pmf$} and 
$\Imc_k \cap \Imc_l\models \Tmc$} \label{line:special}
\State Append $\Imc_k\cap\Imc_l$ to $\Pmf$ \label{line:add2}
\State $\Pmf':=$\Call{UpdatePositiveExamples}{ $  \Kmc, \Pmf,\Lmf$ }  \label{line:rec-call}
\EndIf
\State \textbf{return} ($\Pmf'$)	
\end{algorithmic}
\caption{\textbf{Function} UpdatePositiveExamples ($ \Kmc, \Pmf,\Lmf$)  
\label{alg:remove}}
\end{algorithm}

By Lemma~\ref{lem:positive1} above we have that the `$Y$' consequents of  mvd clauses 
returned by Function `BuildClauses' (Algorithm~\ref{alg:build}) 
are non-empty and mutually disjoint. So the number of mvd clauses 
returned by this function is bounded by $|V|$. 
Regarding Function `UpdatePositiveExamples' (Algorithm~\ref{alg:remove}) called in Line~\ref{line:updateseq}, we note that 
$\Kmc = \Jmc$ is a negative example and 
that in Line~\ref{line:add2}, we have that $\Imc_k\cap\Imc_l\not\models$\Call{BuildClauses}{$\Kmc,\Pmf$}. 
Then, the next lemma ensures that in each recursive call of Function 
`UpdatePositiveExamples' (Algorithm~\ref{alg:remove}) the 
 number of mvd clauses returned by Function `BuildClauses' (Algorithm~\ref{alg:build}) with $\Kmc$ and \Pmf as input, 
 strictly decreases. 
 Since (by Lemma~\ref{lem:positive1}  above) the number of mvd clauses returned by Function `BuildClauses' (Algorithm~\ref{alg:build}) is 
 at most $|V|$, the next lemma bounds  the number of recursive 
calls of Function `UpdatePositiveExamples' (Algorithm~\ref{alg:remove}) to $|V|$.
 
 \begin{lemma} \label{lem:remove1}
Let $\Imc$ be a negative example. 
If  $\Pmc\not\models$ \Call{BuildClauses}{$\Imc, \Pmf$} then 
the number of mvd clauses returned by \Call{BuildClauses}{$\Imc, \Pmf\cdot \Pmc$} 
is strictly smaller than the number of mvd clauses returned by \Call{BuildClauses}{$\Imc, \Pmf$}. 
\end{lemma}

\begin{proof} 
Suppose that ${\sf true}(\Imc) \rightarrow Y_i \vee Z_i \in$ \Call{ BuildClauses}{$\Imc, \Pmf $} is violated by $\Pmc$.
Then, there is $v \in Y_i$ and $w \in Z_i$ such that $v,w \in {\sf false}(\Pmc)$. 
By Lemma~\ref{lem:positive1} there is ${\sf true}(\Imc) \rightarrow Y_j \vee Z_j \in$ \Call{ BuildClauses}{$\Imc, \Pmf $} 
such that $w \in Y_j$ and $v \in Z_j$. In Line~\ref{line:replace-function}, 
Algorithm~\ref{alg:build} replaces (at least) these two mvd clauses by 
a single mvd clause. 
So the number of mvd clauses strictly decreases, as required. 
\end{proof}

By Lemma \ref{lem:disj} below if any two interpretations $\Imc_i,\Imc_j \in \Lmf$ 
violate the same mvd clause in \Tmc then their sets of false variables are mutually disjoint. 
By construction of $\Hmc_0$ we can assume that their sets of false variables are non-empty.  
Then, the number of 
interpretations violating any mvd clause in \Tmc is bounded by $|V|$.  
\begin{lemma}\label{lem:disj}
Let $\Imc_i,\Imc_j\in \Lmf$ and assume $i < j$.
At the end of each iteration, if $\Imc_i,\Imc_j\in\Lmf$ violate $c \in \Tmc$ then 
${\sf false}(\Imc_i) \cap {\sf false}(\Imc_j) = \emptyset$.
\end{lemma}
\begin{corollary}\label{cor:length}
At the end of each iteration every $c \in \Tmc$ is violated by at most $|V|$ interpretations in $\Lmf$.
\end{corollary}
So, at all times the number of elements in \Lmf is bounded by $|\Tmc|\cdot|V|$.  
We now show that   
the number of iterations of Algorithm~\ref{alg:learner} 
is polynomial in the size of \Tmc.
We first present in Lemma~\ref{lem:remove-bound} a polynomial upper bound on the number of iterations  
 where Algorithm~\ref{alg:learner} receives a negative counterexample.  
Note that we obtain this upper bound  
even though the learner does not know the size $|\Tmc|$ of the target.
Lemma~\ref{lem:remove-bound} requires the following technical lemma. 

\begin{lemma}\label{lem:previous} 
In Line \ref{line:remove-set} of Algorithm \ref{alg:learner}, the following holds:
\begin{enumerate}
\item 

 if $\Imc_j $ is removed after the replacement of some $\Imc_i \in \Lmf$ by 
$ \Jmc$ (Line  \ref{line:replace}) then ${\sf false}(\Imc_i)\cap {\sf false}(\Imc_j) = \emptyset$ 
($\Imc_i$ before the replacement);

\item 

 if $\Imc_j,\Imc_k $ with $j < k$ are removed after the replacement of some $\Imc_i \in \Lmf$ by 
$\Jmc$ (Line~\ref{line:replace}) then ${\sf false}(\Imc_j)\cap {\sf false}(\Imc_k) = \emptyset$.

\end{enumerate}

\end{lemma}
 
 \begin{lemma}\label{lem:remove-bound}
Let $N$ be   $ |V|^2\cdot |\Tmc|$.  
The expression $E =  |\Lmf| + (N -    
\sum_{\Imc \in \Lmf} |{\sf false}(\Imc)| )$ always evaluates to a 
natural number inside the loop body and decreases on every iteration 
where Algorithm~\ref{alg:learner} receives a negative counterexample.
\end{lemma}

\begin{proof}
By Corollary \ref{cor:length}, the size of \Lmf is bounded at all times by 
$|V|\cdot |\Tmc|$. Thus, by Corollary~\ref{cor:length}, 
$N$ is an upper bound for $\sum_{\Imc \in \Lmf} |{\sf false}(\Imc)|$. If 
a negative counterexample is received then there are three  possibilities:   
(1)  an element $\Imc$ is appended to \Lmf. Then, $|\Lmf|$ increases by one but $|{\sf false}(\Imc)| \geq 2$ and, 
therefore, $E$ decreases; (2) an element is replaced and no element is removed. Then, $E$ trivially decreases. 
Otherwise, (3) we have that an element $\Imc_i$ is replaced and $p$ interpretations 
are removed from \Lmf in Line \ref{line:remove-set} of Algorithm \ref{alg:learner}. 
By Point 2 of Lemma~\ref{lem:previous}, if $\Imc_i$ is replaced by $ \Jmc$ and 
$\Imc_j,\Imc_k$ are removed then ${\sf false}(\Imc_j) \cap {\sf false}(\Imc_k) = \emptyset$. 
This means that if $p$ interpretations are removed then their sets of false variables are all mutually disjoint. 
By Point 1 of Lemma \ref{lem:previous}, if $\Imc_i$ is replaced by $\Jmc$ and some $\Imc_j$ 
is removed then ${\sf false}(\Imc_j) \cap {\sf false}(\Imc_i) = \emptyset$. Then, 
the $p$ interpretations also have sets of false variables disjoint from ${\sf false}(\Imc_i)$. 
For each interpretation $\Imc_j$ removed we have ${\sf false}(\Imc_j)\subseteq {\sf false}(\Jmc )$ 
(because $\Imc_j \not\models$\Call{BuildClauses}{$\Jmc $,\Pmf}. 
Then, the number of `falses' is at least as large as before. However $|\Lmf|$ 
decreases and, thus, we can ensure that $E$ decreases.
\end{proof}

By Lemma \ref{lem:remove-bound} 
 the total number of iterations  
where Algorithm~\ref{alg:learner} receives a negative counterexample 
is bounded by $N = |V|^2\cdot |\Tmc|$. 
It remains to show a polynomial bound on the total number of 
iterations where Algorithm~\ref{alg:learner} receives a positive counterexample. 
By Corollary \ref{cor:length}, the size of \Lmf is bounded at all times by 
$|V|\cdot |\Tmc|$. By Lemma \ref{lem:positive1}, the number of clauses induced by 
each $\Imc_i \in \Lmf$ is bounded by $|V|$. 
This means that the size of \Hmc is bounded at all times by 
 $ N$. 
If a positive counterexample is received then, by Lemma~\ref{lem:remove1}, 
the size of \Hmc strictly decreases. 
So after giving at most 
$|\Hmc| \leq N$  positive examples 
 the oracle is forced to give a negative example. 
Since the number of negative counterexamples received is also bounded 
by $N$, the total number of iterations 
where Algorithm~\ref{alg:learner} receives a positive counterexample 
is bounded by $N^2$.  

\begin{theorem}\label{th:interpretations}
The problem of learning MVDF from interpretations, more precisely, the learning framework $\Fmf\text{(MVDF$_\Imc$)}$, 
is polynomial time exactly learnable.
\end{theorem}

\subsection{An Example Run}

We describe an example run of Algorithm \ref{alg:learner}. 
In this example, 
 if Function `BuildClauses' (Algorithm~\ref{alg:build}) returns as output 
  mvd clauses of the form $X\rightarrow Y \vee Z$ and $X\rightarrow Z \vee Y$ then  we 
  write only one of  them.  
Suppose that our target  MVDF   is: 
 $$\Tmc= \{2345\rightarrow 1,\quad 123\rightarrow 4\vee 5,\quad 235\rightarrow 1 \vee 4,\quad 2\rightarrow 3\vee 145 \}.$$ 
Initially, the sequence \Pmf of positive examples and the sequence \Lmf of negative examples  are  both empty. 
In Line~\ref{line:hypo0} of Algorithm \ref{alg:learner}, we construct $\Hmc_0 = \{2345\rightarrow~1\}$. 
Suppose that the counterexample to our first  equivalence query with $\Hmc =\Hmc_0$ as input is 
the negative example
 $\Imc_1$, with ${\sf true}(\Imc_1) = \{1,2,3\}$ (note that $\Imc_1$ violates the second mvd clause in \Tmc). 
 Since \Lmf is empty, 
 Algorithm \ref{alg:learner} simply appends $\Imc_1$ to $\Lmf$. 
 In Line \ref{line:transform}, Algorithm~\ref{alg:learner} calls   Function `BuildClauses' (Algorithm~\ref{alg:build}) 
 with $\Imc_1,\Pmf$ as input and receive  
 $\{123\rightarrow 4\vee 5  \}$ 
  as output. At this moment, \Pmf, \Lmf and \Hmc are as follows.  
$$\Pmf =\emptyset \quad\Lmf=\{ \Imc_1\} \quad \Hmc=\{2345 \rightarrow 1 , 123\rightarrow 4\vee 5\}$$
Suppose that the counterexample to our second equivalence query with \Hmc as input is 
 $\Imc_2$, with ${\sf true}(\Imc_2) = \{2, 3, 5\}$. 
 Since $\Imc_2 \cap \Imc_1$ satisfies $\Tmc$, 
 the pair $(\Imc_2, \Imc_1)$ is not a ${\sf goodCandidate}$. So 
 Algorithm~\ref{alg:learner} appends $\Imc_2$ to $\Lmf$. 
   In Line \ref{line:transform},  Algorithm \ref{alg:learner} calls   Function `BuildClauses' (Algorithm~\ref{alg:build}) 
 with $\Imc_1,\Pmf$ and $\Imc_2,\Pmf$ as inputs. 
   We have that \Pmf, \Lmf and \Hmc are  as follows.  
 $$\Pmf=\emptyset \quad\Lmf=\{ \Imc_1, \Imc_2\} \quad \Hmc = \{2345 \rightarrow 1 , 123\rightarrow 4\vee 5 , 235 \rightarrow 1 \vee 4 \}$$
Now assume that the next counterexample is 
  $\Imc_3$, with ${\sf true}(\Imc_3) = \{2, 4\}$. 
In Line \ref{line:cm}, Algorithm \ref{alg:learner} calls Function `RefineCounterexample' 
(Algorithm~\ref{alg:learning-mvdf-up})
  with  $\Imc_3$ and $\Lmf$ as input 
and verifies that the pair $(\Imc_3, \Imc_1)$ is a ${\sf goodCandidate}$.  
The return of Function `RefineCounterexample' (Algorithm~\ref{alg:learning-mvdf-up}) is 
$\Jmc = (\Imc_3 \cap \Imc_1)$. 
In Line \ref{line:cond}, Algorithm \ref{alg:learner} verifies that 
$\Imc_1$ is the first element in \Lmf such that 
$(\Imc_1, \Jmc)$ 
is a ${\sf goodCandidate}$. 
 Then,  Algorithm \ref{alg:learner} calls Function `UpdatePositiveExamples' (Algorithm \ref{alg:remove}) 
 with  $\Kmc =  \Jmc$ (note that ${\sf true}(\Kmc)=\{2\}$), $\Pmf$ and $\Lmf$ as input. 
 We have that 
 $$\text{\Call{BuildClauses}{$\Kmc, \emptyset$}} 
 = \{2\rightarrow 1\vee 345, 2\rightarrow 3\vee 145, 2\rightarrow 4\vee 135, 2\rightarrow 5\vee 134\}.$$
 As $(\Imc_1 \cap \Imc_2) \not\models$  \Call{BuildClauses}{$\Kmc, \emptyset$} and 
 $(\Imc_1 \cap \Imc_2)\models \Tmc$, the condition in Line~\ref{line:special} 
 of Function `UpdatePositiveExamples' (Algorithm \ref{alg:remove}) is satisfied. 
 Then, Function `UpdatePositiveExamples' appends $\Imc_1 \cap \Imc_2$  to \Pmf and 
 makes a recursive call with 
 $\Kmc$, $\Pmf$ and $\Lmf$ as input. 
 Now, 
 $$\text{\Call{BuildClauses}{$\Kmc, \{\Imc_1 \cap \Imc_2\}$}} 
 = \{2\rightarrow 145\vee 3 \},$$
  and, so, $(\Imc_1 \cap \Imc_2)\models$ \Call{BuildClauses}{$\Kmc, \{\Imc_1 \cap \Imc_2\}$}. 
  The output of Function `UpdatePositiveExamples' (Algorithm \ref{alg:remove}) 
  is $\{\Imc_1 \cap \Imc_2\}$. 
  In Line~\ref{line:replace}, Algorithm~\ref{alg:learner} replaces $\Imc_1\in\Lmf$ by~$\Jmc$. 
     In Line~\ref{line:transform},  Algorithm \ref{alg:learner} calls   
     Function `BuildClauses' (Algorithm~\ref{alg:build}) 
 with $\Jmc,\Pmf$ and $\Imc_2,\Pmf$ as inputs.
  Now, \Pmf, \Lmf and \Hmc are  as follows.  
  $$\Pmf =\{  \Imc_1 \cap \Imc_2  \} \quad \Lmf=\{ \Jmc, \Imc_2\} \quad
  \Hmc=\{2345 \rightarrow 1 , 2\rightarrow 145\vee 3 , 235\rightarrow 1\vee 4\} $$
Now assume that 
the counterexample to our fourth  equivalence query with \Hmc as input is 
the negative example~$\Imc_4$,  with ${\sf true}(\Imc_4) = \{1,2,3\}$. 
Function `RefineCounterexample' (Algorithm~\ref{alg:learning-mvdf-up}) returns $\Imc_4$. 
Since there is no $\Imc \in \Lmf$ such that $(\Imc,\Imc_4)$ is 
a ${\sf goodCandidate}$, Algorithm \ref{alg:learner} appends $\Imc_4$ to \Lmf. 
In Line~\ref{line:transform} of  Algorithm \ref{alg:learner}   
   \Pmf, \Lmf and \Hmc are  as follows.  
$$\Pmf=\{ \Imc_1 \cap \Imc_2 \} \quad \Lmf=\{ \Jmc,\Imc_2, \Imc_4\}$$
$$\Hmc=\{2345 \rightarrow 1 , 2\rightarrow 145\vee 3   
   , 235\rightarrow 1\vee 4 , 123\rightarrow 4\vee 5  \}$$
 
 We now have that $\Hmc\equiv\Tmc$ and the learner succeeded.

\section{Reductions among Learning Problems}\label{sec:reduction}
 
A substitution-based technique for problem reductions among boolean formulas 
was  presented by~\cite{kearns1987learnability}. \cite{pitt1988reductions}  
  define a general type of problem reduction 
that preserves polynomial time prediction. This notion was extended  by~\cite{angluin1995won} 
to allow membership queries. In this section, we present 
a notion of reduction that is suitable for the exact learning model with 
membership and equivalence queries. 
It extends a notion of reduction given by \cite{konevexact}.
We then use 
this notion 
to show the reductions  in Figure~\ref{fig:reductions}.

Suppose that $P$ is the problem of exactly learning the framework $\Fmf = (X,\Lmc,\mu)$ and 
 $P'$ is the problem of exactly learning  the framework $\Fmf'= (X',\Lmc,\mu')$. 
Since \Lmc is the same for $\Fmf$ and $\Fmf'$, every correct conjecture used to 
solve $P'$ is also an answer for $P$ and vice-versa. 
One can then reduce $P$ to $P'$ by: 
(a) transforming queries posed  
to oracles ${\sf MEM}_{l,\Fmf'}$ and ${\sf EQ}_{l,\Fmf'}$ into queries 
for the oracles ${\sf MEM}_{l,\Fmf}$ and ${\sf EQ}_{l,\Fmf}$; and 
(b) transforming answers given by the oracles ${\sf MEM}_{l,\Fmf}$ and ${\sf EQ}_{l,\Fmf}$ 
into answers that the oracles ${\sf MEM}_{l,\Fmf'}$ and ${\sf EQ}_{l,\Fmf'}$ 
would provide, where $l \in \Lmc$ is the learning target. 
 For our purposes, we want reductions where  
(i) the frameworks 
use the same target concept representation (as described above) and 
 (ii) preserve polynomial time exact learnability. 
We say that a learning framework $\Fmf=(E, \Lmc, \mu)$
 \emph{polynomial time reduces} to $\Fmf' = (E', \Lmc, \mu')$ if, 
for some  polynomials $p_1(\cdot)$, $p_2(\cdot,\cdot)$ and $p_3(\cdot,\cdot)$ 
there exist 
a function $\f: \Lmc\times E'\to \{\text{ `yes'},\text{ `no'}\}$, translating 
a $\Fmf'$ membership query to $\Fmf$,
and a partial function  
$\g: \Lmc\times\Lmc\times E\to E'$, defined for every $(l,h,e)$ such that
$|h| \leq p_1(|l|)$, translating an answer
to an \Fmf equivalence query to $\Fmf'$,
for which the following conditions hold: 
\begin{itemize} 
\item for all $e'\in E'$ we have $e'\in\mu'(l)$ iff $\f(l,e') = \text{ `yes'}$; 
\item for all $e\in E$ we have $e \in \mu(l)\oplus \mu(h)$ iff 
$\g(l,h,e) \in \mu'(l)\oplus \mu'(h)$;
\item $\f(l,e')$ and $\g(l,h,e)$ are computable in time $p_2(|l|,|e'|)$ and $p_3(|l|,|e|)$, respectively,   
and $l$ can only be accessed by calls to the 
membership oracle ${\sf MEM}_{l,\Fmf}$. 
\end{itemize}

Note that even though $\g$ takes $h$ as input, the polynomial
time bound on computing $\g(l, h, e)$ does not depend on
the size of $h$ as $\g$ is only defined for $h$ polynomial in the
size of $l$.

\begin{theorem}\label{th:reduction}
Let $\Fmf = (E, \Lmc, \mu)$ and $\Fmf' = (E',\Lmc, \mu')$ be   learning frameworks. If there exists a polynomial
time reduction from $\Fmf$ to  $\Fmf'$ and
 $\Fmf'$ is polynomial time exactly learnable then
$\Fmf$ is polynomial time exactly learnable.
\end{theorem}

In the following  
we use Theorem \ref{th:reduction}  to prove that 
MVD can be learned in polynomial time from data relations. 
 The remaining reductions presented in Figure~\ref{fig:reductions} are given in the appendix.



\subsection{Learning MVD from Data Relations}\label{subsec:mvd-red}

\paragraph{\upshape\textbf{ Notation }}
 
A \emph{relation scheme} $V = \{A_1 , \ldots, A_n\}$ is a finite set of symbols, 
called attributes, where 
each attribute $A_i \in V$ is associated with a domain ${\sf dom}(A_i)$ of values. 
 A \emph{tuple} $t$ over $V$ is an element of 
 ${\sf dom}(A_1)\times \ldots \times {\sf dom}(A_n)$.  
A \emph{relation} $r$ (over $V$) is a set of tuples over $V$. 
Given   $S\subseteq V$, let $t[S]$ denote the restriction of a tuple $t$ over $V$ on $S$. 
For example, if the relation scheme is 
${\tt PERSON = \{NAME, BOOK, PET\}}$ and $t={\tt (Alice, Hamlet, Dog)}$ is a tuple over 
${\tt PERSON}$ then $t[{\tt\{NAME, PET\}}] ={\tt(Alice, Dog)}$.
Let $X$, $Y$ and $Z$ be pairwise disjoint subsets of $V$ with $X \cup Y \cup Z = V$. 
We write $xyz$ for a tuple  $t$ over $V$ with $t[X]=x$, $t[Y]=y$ and $t[Z]=z$.  
A \emph{multivalued dependency} (for short \emph{mvd}) 
$X\rightarrow Y\vee Z$
 holds in $r$  
if, and only if, for each two tuples $xyz,xy'z' \in r$ we have that 
$xy'z \in r$ (and, by symmetry, $xyz' \in r$) \footnote{
The standard notation used for mvds is $X \mvd Y\mid Z$ (or $X \mvd Y$) \cite{Fagin1977}. 
However, for the purpose of showing a reduction from MVD$_\Rmc$ to MVDF$_\Imc$, 
it is useful to adopt a uniform representation between 
the two classes. 
}. 
That is, if $t,t'$ are distinct tuples in $r$ with $t[X]=t'[X]$ then we can exchange 
the $Y$ values of $t,t'$ to obtain two tuples that must also be in $r$. 
If \Tmc is a set of mvds over $V$ and, for all $m \in \Tmc$, $m$ holds in $r$ (over $V$) then 
we say that $\Tmc$ holds in $r$. 
We formally define the learning framework $\Fmf\text{(MVD$_\Rmc$)}$ as 
  $(E_\Rmc,\Lmc_{\sf M},\mu_\Rmc)$, where $E_\Rmc$ 
is the set of all relations $r$ over a relation scheme $V$, 
$\Lmc_{\sf M}$ is the set of all sets of mvds that can be expressed with symbols in $V$ and, for every $\Tmc \in \Lmc_{\sf M}$, 
$\mu_\Rmc(\Tmc) = \{ r \in E_\Rmc\mid   \Tmc \text{ holds in } r\}$. 

\smallskip

We now show that $\Fmf\text{(MVD$_\Rmc$)}$ polynomial time 
reduces to $\Fmf\text{(MVDF$_\Imc$)}$. 
To reduce the problem, 
we use the learning algorithm for $\Fmf\text{(MVDF$_\Imc$)}$ 
as a `black box' and: 
(1) transform the membership queries, which come as interpretations into relations; and 
(2) transform counterexamples given by equivalence queries, 
which come as relations into interpretations. 
 
\begin{lemma}\label{lem:mem}
Let $\Fmf\text{(MVD$_\Rmc$)} = (E_\Rmc,\Lmc_{\sf M},\mu_\Rmc)$ and 
$\Fmf\text{(MVDF$_\Imc$)} = (E_\Imc,\Lmc_{\sf M},\mu_\Imc)$ be, respectively,  
the frameworks for learning MVD from relations   and     learning  
MVDF from interpretations. Let $\Tmc \in \Lmc_{\sf M}$ be the target. 
For any interpretation $\Imc \in \mu_\Imc(\Tmc)$, one can construct 
in polynomial time in $|\Tmc|$ a relation $r$ such that 
$r \in \mu_\Rmc(\Tmc)$ if, and only if, $\Imc \in \mu_\Imc(\Tmc)$. 
\end{lemma}

\begin{proof}
Given an interpretation \Imc in $V$, 
we define a pair $p$ of tuples $\{t,t'\}$ over $V$ such that, for each 
  $\gamma \in V$, $t[\gamma] = t'[\gamma]$ if, and only if, 
   $\gamma \in {\sf true}(\Imc)$. 
 By definition of $p$, we have that, for any $m \in \Tmc$, $m$ does not hold in $p$ if, 
 and only if, \Imc violates $m$.  
 Then, $p \in \mu_\Rmc(\Tmc)$ if, and only if, $\Imc \in \mu_\Imc(\Tmc)$. 
\end{proof}

 The close connection between   database relations 
 and propositional logic interpretations was first pointed out by~\cite{fagin1977functional}  and its use in 
a learning theory context appears in \cite{LavinPuente2011}.
To show Lemma \ref{lem:eq} we use the following technical lemma, given by~\cite{Sagiv1981}.  

\begin{lemma}[\cite{Sagiv1981}]\label{lem:sagiv}
Assume that $r$ is a relation over $V$, $\Tmc$ is a set of mvds and $m$ is an mvd 
(both expressed in $V$). 
Suppose that $\Tmc$ holds in $r$ but $m$ does not hold in $r$. Then $r$ 
has a pair $p$ of tuples for which $\Tmc$ holds in $p$ and 
 $m$ does not hold in $p$. 
\end{lemma}

\begin{lemma}\label{lem:eq}
Let $\Fmf\text{(MVD$_\Rmc$)} = (E_\Rmc,\Lmc_{\sf M},\mu_\Rmc)$ and 
$\Fmf\text{(MVDF$_\Imc$)} = (E_\Imc,\Lmc_{\sf M},\mu_\Imc)$ be, respectively,  
the frameworks for learning MVD from relations   and     learning  
MVDF from interpretations.
Let $\Tmc \in \Lmc_{\sf M}$
be the target 
and $\Hmc \in \Lmc_{\sf M}$ be the hypothesis. 
If $r \in \mu_\Rmc(\Tmc)\oplus \mu_\Rmc(\Hmc)$ then one can construct 
in polynomial time in $|\Tmc|$ and $|r|$ an interpretation \Imc such that 
$\Imc \in \mu_\Imc(\Tmc)\oplus \mu_\Imc(\Hmc)$. 
\end{lemma}

\begin{proof}
Assume that $r \in \mu_\Rmc(\Tmc)\oplus \mu_\Rmc(\Hmc)$  is a positive counterexample 
(the case when $r$ is a negative counterexample is analogous).
If $r \not\in \mu_\Rmc(\Hmc)$ then there is   $m \in \Hmc$ such that 
$m$ does not hold in $r$. By Lemma \ref{lem:sagiv}, $r$ has 
a pair $p$ of tuples for which $\Tmc$ holds in $p$ and 
 $m$ does not hold in $p$. Then, $p \in \mu_\Rmc(\Tmc)\setminus \mu_\Rmc(\Hmc)$. 
One can find $p \subseteq r$, 
by simply checking, for all possible pairs $p$ of tuples in $r$, 
whether $\Hmc$ does not hold in $p$ and (with membership queries) whether $\Tmc$ holds in $p$. 
Once $p = \{t,t'\}$ is computed, we define $\Imc$ such that 
${\sf true}(\Imc) = \{\gamma \in V \mid t[\gamma] = t'[\gamma] \}$. 
 By definition of $\Imc$, 
 we have that, for any $m' \in \Tmc\cup \Hmc$, $m'$ does not hold in $p$ if, 
 and only if, \Imc violates $m'$.  
 Then, $\Imc \in \mu_\Imc(\Tmc)\oplus \mu_\Imc(\Hmc)$. 
\end{proof}

Lemma \ref{lem:mem} shows how one can compute 
$\f$ (described in Definition \ref{def:reduction2}) with 
$p_2(|\Tmc|,|\Imc|) = k\cdot|\Imc|$ steps, for some constant $k$.  
Lemma \ref{lem:eq} shows 
how one can compute  $\g$ in $p_3(|\Tmc|,|r|) = k\cdot|r|^2$, for some constant $k$. 
Also, we have seen in Section \ref{sec:mvdf-interpretation} that the size of the hypothesis $\Hmc$ 
computed by Algorithm~\ref{alg:learner} is bounded by 
$|V|\cdot|\Tmc|$. 
Then,
  $p_1(|\Tmc|) = |V|\cdot|\Tmc|$.
Using Theorems \ref{th:interpretations} and \ref{th:reduction} we can state the following.   

\begin{theorem} 
The problem of learning MVD from relations, more precisely, the learning framework $\Fmf\text{(MVD$_\Rmc$)}$, 
is polynomial time exactly learnable.
\end{theorem}

\section{Discussion}


We  solved the open question raised by~\cite{LavinPuente2011}, 
showing a polynomial time algorithm that exactly learns the class MVDF 
from interpretations. 
From a 
database design perspective, a transformation of our algorithm can be used to extract 
 multivalued dependencies 
from examples of relations. This process is a sort of knowledge discovery, 
which can help in restructuring databases and  finding data dependencies 
that database designers did not foresee. 
From a  theoretical point of view, 
we 
take a step towards 
identifying 
important concept classes %
that can be learned in polynomial time,  
a natural 
research topic in computational learning theory. 
However, it remains open the question of whether the class MVDF can be 
exactly learned in polynomial time from entailments (where the entailments are mvd clauses). 
We know that, for propositional Horn, learning 
from entailments reduces to learning from interpretations. 
However, for MVDF a similar reduction is not so easy. 
The main obstacle is the transformation of membership queries, where one needs to
 decide whether an 
interpretation is a model of the target using polynomially many entailment queries.

\bibliographystyle{plain}
\bibliography{ourbib}

\appendix

\section{Proofs for Section \ref{sec:mvdf-interpretation}}

\normalsize 

We provide the proofs for the lemmas  stated in Section \ref{sec:mvdf-interpretation}. 
We note that our algorithm maintains a sequence of positive examples, as in \cite{Arias2011}. Also,  
the construction of  mvd clauses in the hypothesis is inspired by \cite{Puente15}.

\begin{remark}\upshape\label{remark1}
In our proof we only consider interpretations \Imc  such that $|{\sf false}(\Imc)|\geq 2$. 
This is justified by the fact that in Line \ref{line:hypo0} of Algorithm \ref{alg:learner}
we check whether $\Tmc\models V \rightarrow \mathbf{F}$ and 
whether $\Tmc\models V\setminus\{v\} \rightarrow v$, for all $v \in V$,
and if so we add them to $\Hmc_0$ (note that this can be easily checked with queries 
to ${\sf MEM}_{\Tmc,\Fmf\text{(MVDF$_\Imc$)}}$). 
Any negative counterexample \Imc received by Algorithm \ref{alg:learner} 
is such that 
$|{\sf false}(\Imc)|\geq 2$ and it can only violate mvd clauses $X \rightarrow Y \vee Z \in \Tmc$ with $Y$ and $Z$ non-empty. 
Also, any positive counterexample can only violate mvd clauses $X \rightarrow Y \vee Z \in \Hmc$ with $Y$ and $Z$ non-empty.
\end{remark}
We consistently use \Pmf and \Lmf for the sequences of positive and negative 
examples of Algorithm~\ref{alg:learner}, respectively.
Before we show Lemma~\ref{lem:positive1} we need the following technical lemma. 

\begin{lemma}
\label{lem:positive3}
Let $\Imc$ be a negative example for $\Tmc$  that covers $X \rightarrow Y \vee Z \in \Tmc$.
Let \Call{BuildClauses}{$\Imc,$ $ \Pmf$} be the set  $\{ {\sf true}(\Imc) \rightarrow Y_1 \vee Z_1,
  \ldots ,$ ${\sf true}(\Imc) \rightarrow Y_k \vee Z_k  \}$. 
  Then, for all $i$, $1 \leq i  \leq k$, either $Y_i \subseteq Y$ or $Y_i \subseteq Z$.
\end{lemma}
\begin{proof} 
The proof is by induction on the number of elements in $\Pmf$.
The lemma is true when $\Pmf$ is empty because Function `BuildClauses' 
(Algorithm~\ref{alg:build})
returns the set constructed  in Line~\ref{line:initial}, which contains an 
mvd clause ${\sf true}(\Imc) \rightarrow v \ \vee \ V\setminus ({\sf true}(\Imc)\cup\{v\})$ for each $v \in {\sf false}(\Imc)$.
Now suppose that the lemma holds for $\Pmf$ with $m \in \mathbb{N}$ elements. 
 We show that it holds for $\Pmf$ with $m+1$ elements.
Let \Pmc be a fresh positive example (for \Tmc). 
If $\Pmc\models$ \Call{BuildClauses}{$\Imc, \Pmf$} then 
\Call{BuildClauses}{$\Imc, \Pmf$} $=$ \Call{BuildClauses}{$\Imc, \Pmf\cdot \Pmc$}. So, 
 by induction hypothesis the lemma holds. 

Otherwise, $\Pmc\not\models$ \Call{BuildClauses}{$\Imc, \Pmf$}. 
Let 
${\sf true}(\Imc)  \rightarrow Y_{1}  \vee Z_{1},  \ldots ,$ ${\sf true}(\Imc)  \rightarrow Y_{k} \vee Z_{k}$ 
be the mvd clauses in \Call{BuildClauses}{$\Imc, \Pmf$} violated by $\Pmc$. 
These mvd clauses are replaced, in \Call{BuildClauses}{$\Imc, \Pmf\cdot \Pmc$}, by
 $  {\sf true}(\Imc) \rightarrow \bigcup_{j=1}^k Y_{j} \vee 
 (\bigcup_{j=1}^k Z_{j}\setminus \bigcup_{j=1}^k Y_{j})$. 
 So we need to show that either $  \bigcup_{j=1}^k Y_{j} \subseteq Y $ or $  \bigcup_{j=1}^k Y_{j} \subseteq Z  $. 
 As $\Pmc$ violates these mvd clauses, 
 we have that ${\sf true}(\Imc) \subseteq {\sf true}(\Pmc)$  
 and $\Pmc$ must have some zero in $Y_{j}$ for all $1 \leq j \leq k$. 
 Also, since $\Pmc$ is a positive example and $X \subseteq {\sf true}(\Imc)$ either 
 ${\sf false}(\Pmc) \subseteq Y$ or  ${\sf false}(\Pmc) \subseteq Z$. 
 Therefore, either (a) each $Y_{j}$ has at least one variable in $Y$ or 
   (b) each $Y_{j}$ has at least one variable in $Z$. 
   In case (a), 
   by induction hypothesis, either 
   $Y_{j} \subseteq Y$ or $Y_{j} \subseteq Z$. 
   As $Y \cap Z = \emptyset$, $Y_{j} \subseteq Y$ for all $1 \leq j \leq k$.  
      Therefore 
   $  \bigcup_{j=1}^k Y_{j} \subseteq Y $. 
One can prove in the same way that in 
  case (b)  we have $  \bigcup_{j=1}^k Y_{j} \subseteq Z  $.  
\end{proof}

 \smallskip
   
\noindent
{\bf Lemma~\ref{lem:positive1}   (restated).} 
\textit{Let $\Imc$ be a negative example. 
Let  \Call{BuildClauses}{$\Imc, \Pmf$}
$=\{ 
 {\sf true}(\Imc) \rightarrow Y_1 \vee Z_1,$ $ \ldots , {\sf true}(\Imc) \rightarrow Y_k \vee Z_k\} $.  
Then, for all $i, j$, such that $1 \leq i < j \leq k$, 
we have that $(\ast)$:
$Y_i \cap Y_j = \emptyset \text{ and }  
\bigcup_{j=1}^k Y_j  = {\sf false}(\Imc).$ 
Moreover, for any ${\sf true}(\Imc)\rightarrow Y_i\vee Z_i$, $1\leq i \leq k$, 
we have that $Y_i$, $Z_i$ are non-empty.}

\smallskip
\noindent

\begin{proof}
The  proof is by induction on the size of $\Pmf$.
The lemma is true when $\Pmf$ is empty because 
Function `BuildClauses' (Algorithm~\ref{alg:build}) 
returns the set constructed  in Line~\ref{line:initial}, 
which contains an mvd clause $ X \rightarrow v \ \vee \ V\setminus (X\cup\{v\})$ 
for each $v \in {\sf false}(\Imc)$, where $X={\sf true}(\Imc)$ 
(note that, as in Remark~\ref{remark1}, $|{\sf false}(\Imc)|\geq 2$ and therefore $V\setminus (X\cup\{v\})$ is 
 non-empty).
 Now suppose that the lemma holds for $\Pmf$ with $m \in \mathbb{N}$ elements. 
 We show that it holds for $\Pmf$ with $m+1$ elements.
Let \Pmc be a fresh positive example. 
If $\Pmc\models$ \Call{BuildClauses}{$\Imc, \Pmf$} then 
\Call{BuildClauses}{$\Imc, \Pmf$} $=$ \Call{BuildClauses}{$\Imc, \Pmf\cdot\Pmc$}. So, 
 by induction hypothesis the lemma holds. 
Otherwise, $\Pmc\not\models$ \Call{BuildClauses}{$\Imc, \Pmf$}. 
 Let $ X  \rightarrow Y_{1} \vee Z_{1} , 
  \ldots , X  \rightarrow Y_{k} \vee Z_{k}$ 
be the mvd clauses in \Call{BuildClauses}{$\Imc, \Pmf$} 
 violated by $\Pmc$. 
 These mvd clauses are replaced, in \Call{BuildClauses}{$\Imc, \Pmf\cdot\Pmc$}, by
 $  X \rightarrow \bigcup_{j=1}^k Y_{j} \vee 
 (\bigcup_{j=1}^k Z_{j}\setminus \bigcup_{j=1}^k Y_{j})$. 
 Clearly, $(\ast)$ holds in \Call{BuildClauses}{$\Imc, \Pmf\cdot\Pmc$}. 
It remains to show that 
 $(\bigcup_{j=1}^k Z_{j}\setminus \bigcup_{j=1}^k Y_{j})$ is not empty. 
 Since \Imc is a negative example, 
 it violates some clause $X'\rightarrow Y' \vee Z' \in \Tmc$ with 
 $Y',Z'$ non-empty (see Remark~\ref{remark1}). 
 Now suppose to the contrary that $(\bigcup_{j=1}^k Z_{j}\setminus \bigcup_{j=1}^k Y_{j})$ is   empty.  
   Then, 
 $\bigcup_{j=1}^k Y_{j}={\sf false}(\Imc)$ and, by Lemma~\ref{lem:positive3}, 
  $\bigcup_{j=1}^k Y_{j}$  is included either in $Y'$ or in $Z'$. 
  If ${\sf false}(\Imc)$ is included either in 
  $Y'$ or in $Z'$ then \Imc does not violate $X'\rightarrow Y' \vee Z'$. 
  This contradicts our assumption that  \Imc violates $X'\rightarrow Y' \vee Z' \in \Tmc$. 
\end{proof}

We now want to show Lemma~\ref{lem:disj}. 
Before we prove Lemma~\ref{lem:disj}, we need   
Lemmas~\ref{lem:new}-\ref{lem:nogood}  below.

\begin{lemma}\label{lem:new} 
Assume that an interpretation $\Imc$ violates $X \rightarrow Y \vee Z \in \Tmc$. 
For all  $\Imc_i \in \Lmf$ such that 
$\Imc_i$ covers $X \rightarrow Y \vee Z$,  
 ${\sf true}(\Imc_i) \subseteq {\sf true}(\Imc)$ if, and only if, $\Imc~\not\models$\Call{ BuildClauses}{$\Imc_i, \Pmf$}. 
\end{lemma}

\begin{proof}
 The ($\Leftarrow$) direction is trivial. Now, suppose that ${\sf true}(\Imc_i) \subseteq {\sf true}(\Imc)$  to prove
($\Rightarrow$). As $\Imc \not\models X \rightarrow Y \vee Z$, we have that 
 $X \subseteq {\sf true}(\Imc)$ 
    and there 
   are $v \in Y$ and $w \in Z$ such that $v,w \in {\sf false}(\Imc)$.  
As ${\sf true}(\Imc_i) \subseteq {\sf true}(\Imc) $, we have that  
$v,w \in {\sf false}(\Imc_i)$. 
  By Lemma~\ref{lem:positive1}, 
there are ${\sf true}(\Imc_i)  \rightarrow Y_1 \vee Z_1,{\sf true}(\Imc_i)  \rightarrow Y_2 \vee Z_2
\in$\Call{ BuildClauses}{$\Imc_i, \Pmf$} such that $v  \in Y_1$  and 
$w  \in Y_2$. By Lemma~\ref{lem:positive3}, $Y_1 \subseteq Y$ and $Y_2 \subseteq Z$. As $Y \cap Z = \emptyset$, 
we have that 
$Y_1 \cap Y_2 = \emptyset$. 
So, $v \in Z_2$ and $w \in Z_1$, which means that \Imc violates 
 both ${\sf true}(\Imc_i)  \rightarrow Y_1 \vee Z_1$ and ${\sf true}(\Imc_i)  \rightarrow Y_2 \vee Z_2$ 
 in \Call{ BuildClauses}{$\Imc_i, \Pmf$}. 
\end{proof}

We can see the hypothesis $\Hmc$ as a sequence of sets of multivalued clauses, 
where each $\Hmc_i$ corresponds to the output of Function `BuildClauses' (Algorithm~\ref{alg:build})
with $\Imc_i \in \Lmf$ and \Pmf as input.

\begin{lemma}\label{lem:curious} 
At the end of each iteration,  $\Imc_i \models \Hmc \setminus \Hmc_i$, for all $\Imc_i \in \Lmf$.  
\end{lemma}

\begin{proof}
Let \Jmc be the interpretation computed in Line~\ref{line:cm} of Algorithm~\ref{alg:learner}. 
If  Algorithm \ref{alg:learner} executes Line~\ref{line:append} then it holds that $\Jmc \models \Hmc$. 
If there is $\Imc_j \in \Lmf$ such that $\Imc_j \not \models$\Call{BuildClauses}{$\Jmc, \Pmf$}  
then ${\sf true}(\Jmc) \subset {\sf true}(\Imc_j)$ and the pair $(\Imc_j, \Jmc)$ is a 
${\sf goodCandidate}$. This contradicts the fact that Algorithm \ref{alg:learner} did not replace some interpretation in \Lmf. 
Otherwise,  Algorithm \ref{alg:learner} executes Lines~\ref{line:updateseq} 
and~\ref{line:replace}, replacing an interpretation $\Imc_i \in \Lmf$ 
by $ \Jmc$, where the pair  $(\Imc_i, \Jmc)$ is a ${\sf goodCandidate}$. 
In this case, by Definition~\ref{def:good} part (ii), 
$\Imc_i \cap \Jmc \models \Hmc$.
It remains to check that for any other $\Imc_j  \in \Lmf$ it holds that 
 $\Imc_j \models$\Call{BuildClauses}{$  \Jmc, \Pmf$}, but 
 this is always true because of Line~\ref{line:remove-set}. 
\end{proof}

We also require the following technical lemma from~\cite{hermo2015exact}. 

\begin{lemma}[\cite{hermo2015exact}]\label{lem:technical}
Let \Tmc be a set of mvd clauses.
If $\Imc$ and $\Jmc$ are interpretations such that $\Imc \models \Tmc$ and
 $\, \Jmc \models \Tmc$, but $\Imc \cap \Jmc \not \models \Tmc$, then
$\mbox{true}(\Imc) \cup \mbox{true}(\Jmc) = V$.
\end{lemma}

\begin{lemma}\label{lem:recursive}
If Algorithm \ref{alg:learner} 
replaces some $\Imc_i \in \Lmf$ with $\Jmc$ then ${\sf false}(\Imc_i)\subseteq {\sf false}(\Jmc)$ 
($\Imc_i$ before the replacement). 
\end{lemma}

\begin{proof}
Suppose to 
the contrary that 
${\sf false}(\Imc_i)\not\subseteq {\sf false}(\Jmc)$.
That is, 
($\ast$) ${\sf true}(  \Jmc \cap \Imc_i)\subset {\sf true}(\Jmc)$.
If Algorithm \ref{alg:learner} replaced $\Imc_i \in \Lmf$ by \Jmc then 
$(\Imc_i,\Jmc)$ is a ${\sf goodCandidate}$. Then, 
$\Imc_i\cap \Jmc\not\models \Tmc$
and $\Imc_i\cap \Jmc\models \Hmc$. 
If (i) ${\sf true}(  \Jmc \cap \Imc_i)\subset {\sf true}(\Jmc)$ (by ($\ast$)), 
(ii) $\Jmc\cap \Imc_i\models \Hmc$ and (iii) 
$\Jmc\cap \Imc_i\not\models \Tmc$; then $(\Jmc,\Imc_i)$ is a ${\sf goodCandidate}$. 
This contradicts the condition in Line~\ref{line:cond2} of Algorithm~\ref{alg:learning-mvdf-up}, 
which would not return $\Jmc$ but make a recursive call with $\Jmc \cap \Imc_i$  
 and, thus, ${\sf false}(\Imc_i)\subseteq {\sf false}(\Jmc)$. 
\end{proof}

\begin{lemma}\label{lem:tech2}
Let $\Imc_i,\Imc_j\in \Lmf$ and assume $i < j$. 
At the end of each iteration, if $c \in \Tmc$ is violated by $\Imc_i, \Imc_j \in \Lmf$ 
then the pair $(\Imc_i, \Imc_j)$ is a ${\sf goodCandidate}$ or ${\sf false}(\Imc_i) \cap {\sf false}(\Imc_j) = \emptyset$. 
\end{lemma}
\begin{proof}
We prove that if  ${\sf false}(\Imc_i) \cap {\sf false}(\Imc_j) \not = \emptyset$, 
then $(\Imc_i, \Imc_j)$ is a ${\sf goodCandidate}$.  
By Lemma \ref{lem:new}, ${\sf true}(\Imc_i)\subseteq {\sf true}(\Imc_j)$ if, 
and only if, $\Imc_i\not\models$ \Call{BuildClauses}{$\Imc_j, \Pmf$}. 
If $\Imc_i$ covers $c \in \Tmc$ and $\Imc_j$ violates $c \in \Tmc$  
then it   
follows from Lemma \ref{lem:curious} that ${\sf true}(\Imc_i)\not\subseteq {\sf true}(\Imc_j)$. 
So (i) ${\sf true}(\Imc_i \cap \Imc_j) \subset {\sf true}(\Imc_i)$. Also 
by Lemma~\ref{lem:curious}, it holds that $\Imc_i \models \Hmc\setminus (\Hmc_i\cup \Hmc_j)$ 
and $\Imc_j \models \Hmc\setminus (\Hmc_i\cup \Hmc_j)$. Now, 
by Lemma~\ref{lem:technical},  ${\sf false}(\Imc_j) \cap {\sf false}(\Jmc) \not = \emptyset$ 
implies that $\Imc_i \cap \Imc_j \models \Hmc\setminus (\Hmc_i\cup \Hmc_j)$. Since 
${\sf true}(\Imc_i \cap \Imc_j) \subset {\sf true}(\Imc_i)$, we actually have that %
(ii) $\Imc_i \cap \Imc_j \models \Hmc$. To finish, we know that (iii) $\Imc_i \cap \Imc_j \not\models \Tmc$ because $c \in \Tmc$ is violated by both $\Imc_i$ and $\Imc_j$. Hence, we obtain the conditions (i), (ii), and (iii) 
of Definition~\ref{def:good}, 
and therefore the pair $(\Imc_i, \Imc_j)$ is a ${\sf goodCandidate}$.
\end{proof}

\begin{lemma}\label{lem:nogood}
Let $\Imc_i,\Imc_j\in \Lmf$ and assume $i < j$. 
At the end of each iteration, the pair $(\Imc_i,\Imc_j)$ is not a ${\sf goodCandidate}$ or 
${\sf false}(\Imc_i)\cap {\sf false}(\Imc_j) = \emptyset$.
\end{lemma}

\begin{proof}
Let \Jmc be a countermodel computed in Line \ref{line:cm} of Algorithm \ref{alg:learner}. 
Consider the possibilities. 
\begin{itemize}
\item If  Algorithm \ref{alg:learner} appends \Jmc to \Lmf, then  for all $\Imc_k \in \Lmf$ the pair $(\Imc_k,\Jmc)$ 
cannot be a ${\sf goodCandidate}$, 
because otherwise the condition in 
Line~\ref{line:cond} would be satisfied and, instead of appending $\Jmc$, 
Algorithm~\ref{alg:learner} would replace some interpretation $\Imc_k \in \Lmf$. 

\item Now assume that Algorithm \ref{alg:learner} replaces (a) $\Imc_i$ by 
$ \Jmc$ or (b) $\Imc_j$ by $\Jmc$. Suppose the lemma fails to hold in case (a). 
The pair $(\Jmc, \Imc_j)$ is a ${\sf goodCandidate}$. 
This contradicts the condition in Line~\ref{line:cond2} of Algorithm~\ref{alg:learning-mvdf-up}, 
which would not return $\Jmc$ but make a recursive call with $\Jmc \cap \Imc_j$. 
Now, suppose the lemma fails to hold in case (b). 
The pair $(\Imc_i , \Jmc )$ is a ${\sf goodCandidate}$. 
This contradicts the fact that in Line~\ref{line:goodcandidate1} of Algorithm~\ref{alg:learner},  
the first ${\sf goodCandidate}$ is replaced and since $i < j$,  $\Imc_i$ should be replaced instead of $\Imc_j$. 
\item It remains to check the case where  
  Algorithm \ref{alg:learner} replaces 
   $\Imc \in \Lmf\setminus \{\Imc_i,\Imc_j\} $ by $\Jmc$. We prove that if at the end of 
   the iteration, the pair $(\Imc_i,\Imc_j)$ is a 
${\sf goodCandidate}$ then ${\sf false}(\Imc_i)\cap {\sf false}(\Imc_j) = \emptyset$. So assume  that  (i) ${\sf true}(\Imc_i \cap \Imc_j) \subset {\sf true}(\Imc_i)$; (ii) $\Imc_i \cap \Imc_j \models \Hmc$; and (iii)
$\Imc_i \cap \Imc_j \not\models \Tmc$. Point (ii) implies that
$\Imc_i \cap \Imc_j \models \Hmc_i$ and $\Imc_i \cap \Imc_j \models \Hmc_j$. 
Denote by $\Hmc'$ the hypothesis  at 
the beginning of the   iteration. 
By induction hypothesis, before the replacement, 
$(\Imc_i,\Imc_j)$ was not a ${\sf goodCandidate}$ (or ${\sf false}(\Imc_i)\cap {\sf false}(\Imc_j) = \emptyset$
and we are done). 
Therefore, $\Imc_i \cap \Imc_j \not\models \Hmc'$, and there is $\Hmc'_k$ such that 
$\Imc_i \cap \Imc_j \not\models \Hmc'_k$. We know that $k \not \in \{i,j\}$ because 
$\Hmc_i =  \Hmc'_i$ and  $\Hmc_j =  \Hmc'_j$.
 As $\Imc_j\models \Hmc' \setminus \Hmc'_j$ (by Lemma~\ref{lem:curious}), we have that $\Imc_j\models \Hmc'_k$ . By the same argument $\Imc_i\models \Hmc'_k$. Hence, 
 by Lemma~\ref{lem:technical}, ${\sf false}(\Imc_i)\cap {\sf false}(\Imc_j) = \emptyset$.
\end{itemize}
\end{proof}

We are now ready for Lemma~\ref{lem:disj}. 

 \smallskip
   
\noindent
{\bf Lemma~\ref{lem:disj}  (restated).} 
\textit{
 Let $\Imc_i,\Imc_j\in \Lmf$ and assume $i < j$.
At the end of each iteration, if $\Imc_i,\Imc_j\in\Lmf$ violate $c \in \Tmc$ then 
${\sf false}(\Imc_i) \cap {\sf false}(\Imc_j) = \emptyset$.
}

\smallskip
\noindent

\begin{proof}
On one hand, by Lemma~\ref{lem:tech2} the pair $(\Imc_i, \Imc_j)$ is 
a ${\sf goodCandidate}$ or ${\sf false}(\Imc_i) \cap {\sf false}(\Imc_j) = \emptyset$. 
On the other, by Lemma~\ref{lem:nogood} the pair $(\Imc_i, \Imc_j)$ 
is not a ${\sf goodCandidate}$ or ${\sf false}(\Imc_i) \cap {\sf false}(\Imc_j) = \emptyset$. 
We conclude that ${\sf false}(\Imc_i) \cap {\sf false}(\Imc_j) = \emptyset$. 
\end{proof}

Lemma~\ref{lem:previous}   shows that  (1) if an interpretation $\Imc_i$ 
  is replaced and an element $\Imc_j$ is removed 
  from \Lmf then they are mutually disjoint; and 
   (2) if any two elements are removed then they are mutually disjoint. 
Lemmas \ref{lem:technical2}  and \ref{lem:point-i} below prepare  for the proof of Lemma  \ref{lem:previous}.

\begin{lemma}\label{lem:technical2}
Let \Pmc and \Imc be a positive and a negative example, respectively. 
If $\Pmc\in\Pmf$ then $\Pmc \models$ \Call{BuildClauses}{$\Imc, \Pmf$}. 
\end{lemma}

\begin{proof} 
The proof is by induction on the number of elements in \Pmf. 
In the base case \Pmf is empty, so the lemma holds trivially. 
Now suppose that the lemma holds for $\Pmf$ with $m \in \mathbb{N}$ elements. 
 We show that it holds for $\Pmf$ with $m+1$ elements. 
Let \Pmc be a fresh positive example. 
We first want to show that $\Pmc\models$ \Call{BuildClauses}{$\Imc, \Pmf\cdot\Pmc$}. 
If $\Pmc\models$ \Call{BuildClauses}{$\Imc, \Pmf$} then 
\Call{BuildClauses}{$\Imc, \Pmf$} $=$ \Call{BuildClauses}{$\Imc, \Pmf\cdot\Pmc$}. So, 
 by induction hypothesis, the lemma holds.

Otherwise, $\Pmc\not\models$ \Call{BuildClauses}{$\Imc, \Pmf$}.  
Let 
$X  \rightarrow Y_{1}  \vee Z_{1},  \ldots ,$ $X \rightarrow Y_{k} \vee Z_{k}$ 
be the mvd clauses in \Call{BuildClauses}{$\Imc, \Pmf$} violated by $\Pmc$, where 
${\sf true}(\Imc)= X$. 
These mvd clauses are replaced, in \Call{BuildClauses}{$\Imc, \Pmf\cdot\Pmc$}, by
 $  X \rightarrow \bigcup_{j=1}^k Y_{j} \vee 
 (\bigcup_{j=1}^k Z_{j}\setminus \bigcup_{j=1}^k Y_{j})$. 
 For short denote the latter mvd clause by $X \rightarrow Y' \vee Z'$. 
Suppose to the contrary that  $\Pmc \not\models$ \Call{ BuildClauses}{$\Imc, \Pmf \cdot\Pmc$}. 
By construction of \Call{ BuildClauses}{$\Imc, \Pmf \cdot \Pmc $}, 
the only mvd clause that can be violated by \Pmc is $X \rightarrow Y' \vee Z'$. 
Then, there is $v,w \in {\sf false}(\Pmc)$ such that $v \in Y'$ and $w \in Z'$. 
By definition of $X \rightarrow Y' \vee Z'$,   there is $X \rightarrow Y_i \vee Z_i \in $ 
 \Call{ BuildClauses}{$\Imc, \Pmf$} such that $w \in Z_i$. If 
 $w \in Z_i$ then, by Lemma~\ref{lem:positive1}, there is $X \rightarrow Y_j \vee Z_j \in$ 
 \Call{ BuildClauses}{$\Imc, \Pmf$}  
such that $w \in Y_j$. If $\Pmc \not\models X \rightarrow Y_j \vee Z_j$ then 
this contradicts the fact that  $w \in Z'$.  
Otherwise, $\Pmc \models X \rightarrow Y_j \vee Z_j$. So, ${\sf false}(\Pmc)\subseteq Y_j$ 
and $X \rightarrow Y_j \vee Z_j \in $ \Call{BuildClauses}{$\Imc, \Pmf $}. 
As $v \in Y_j\cap Y'$, we have that $Y_j\cap Y' \neq \emptyset$. This contradicts Lemma~\ref{lem:positive1}. 

 It remains to show that for any other $\Pmc' \in \Pmf$, we  have that 
 $\Pmc'\models$ \Call{BuildClauses}{$\Imc, \Pmf'\cdot\Pmc$}.  
 If $\Pmc'\not\models$ \Call{BuildClauses}{$\Imc, \Pmf'\cdot\Pmc$} then the 
 only clause that can be violated by $\Pmc'$ is $X \rightarrow Y' \vee Z'$. 
 Then, there is $v',w' \in {\sf false}(\Pmc')$ such that $v' \in Y'$ and $w' \in Z'$. 
 Therefore, $v' \in Y_i$, for some $X\rightarrow Y_i \vee Z_i \in$ \Call{BuildClauses}{$\Imc, \Pmf$} 
 violated by \Pmc. 
 If $w' \in Z'$ then, as $Z' = (\bigcup_{j=1}^k Z_{j}\setminus \bigcup_{j=1}^k Y_{j}) = 
 \bigcap_{j=1}^k Z_{j}$, we have that $w' \in Z_i$.  
 Then, $\Pmc'\not\models X\rightarrow Y_i \vee Z_i$. This, contradicts the 
 fact that, by induction hypothesis, $\Pmc'\models$ \Call{BuildClauses}{$\Imc, \Pmf$}. 
\end{proof}

\begin{lemma}\label{lem:point-i}
Let \Imc,\Jmc and \Kmc be negative examples such that 
${\sf true}(\Imc)\subseteq {\sf true}(\Jmc)\subseteq {\sf true}(\Kmc)$. 
If $\Kmc \not\models$ \Call{BuildClauses}{$ \Imc,\Pmf$}
then $\Kmc \not\models$ \Call{BuildClauses}{$ \Jmc,\Pmf$}. 
\end{lemma}

\begin{proof}
If $\Kmc \not\models$ \Call{BuildClauses}{$ \Imc,\Pmf$} then there is 
${\sf true}(\Imc)\rightarrow Y\vee Z  \in $ \Call{BuildClauses}{$ \Imc,\Pmf$} with  
$v \in Y$,  $w \in Z$ such that $v,w \in {\sf false}(\Kmc)$. If 
$v,w \in {\sf false}(\Kmc)$ then $v,w \in {\sf false}(\Jmc)$. 
If there is ${\sf true}(\Jmc)\rightarrow Y'\vee Z'  \in $ \Call{BuildClauses}{$ \Jmc,\Pmf$} with  
$v \in Y'$,  $w \in Z'$ then $\Kmc \not\models$ \Call{BuildClauses}{$ \Jmc,\Pmf$}. 
Otherwise, there is no such mvd clause in \Call{BuildClauses}{$ \Jmc,\Pmf$}. 
This means that there is $\Pmc \in \Pmf$ such that  
${\sf true}(\Jmc)\subseteq {\sf true}(\Pmc)$ and $v,w \in {\sf false}(\Pmc)$. 
 As ${\sf true}(\Jmc)\subseteq {\sf true}(\Pmc)$, we have that 
 ${\sf true}(\Imc)\subseteq {\sf true}(\Pmc)$. Then, 
 $\Pmc\not\models$ \Call{BuildClauses}{$ \Imc,\Pmf$}. 
 Since $\Pmc \in \Pmf$, this contradicts Lemma \ref{lem:technical2}. 
\end{proof}

We can now prove Lemma~\ref{lem:previous}.

 \smallskip
   
\noindent
{\bf Lemma~\ref{lem:previous}  (restated).} 
\textit{ 
In Line \ref{line:remove-set} of Algorithm \ref{alg:learner}, the following holds:}
\begin{enumerate}
\item 

\textit{ if $\Imc_j $ is removed after the replacement of some $\Imc_i \in \Lmf$ by 
$  \Jmc$ (Line  \ref{line:replace}) then ${\sf false}(\Imc_i)\cap {\sf false}(\Imc_j) = \emptyset$ 
($\Imc_i$ before the replacement);}

\item 
\textit{
 if $\Imc_j,\Imc_k $ with $j < k$ are removed after the replacement of some $\Imc_i \in \Lmf$ by 
$  \Jmc$ (Line~\ref{line:replace}) then ${\sf false}(\Imc_j)\cap {\sf false}(\Imc_k) = \emptyset$.
}
\end{enumerate}

\smallskip
\noindent

\begin{proof}
We first argue that if $\Imc_j$ is removed then $i < j$. 
Suppose to the contrary that $j < i$ and $\Imc_j$ is removed 
after the replacement of $\Imc_i$ by \Jmc. 
Then, $\Imc_j\not\models$BuildClauses($\Jmc$), which means that   
${\sf true}(\Jmc) \subset {\sf true}(\Imc_j)$.
We have that (i)  
${\sf true}(\Imc_j \cap \Jmc) \subset {\sf true}(\Imc_j)$;  
 (ii) %
$\Imc_j \cap \Jmc \models \Hmc$ and 
(iii) $ (\Imc_j \cap \Jmc) \not \models \Tmc$ (as $ \Imc_j \cap \Jmc  =~\Jmc$). 
Then, by Definition~\ref{def:good}, the pair $(\Imc_j, \Jmc)$ is a 
${\sf goodCandidate}$. This contradicts the fact that in 
Line~\ref{line:goodcandidate1} of Algorithm~\ref{alg:learner}, 
the first ${\sf  goodCandidate}$ is replaced.

So we can assume that $i < j < k$.
We now argue that under the conditions stated by this lemma  if 
${\sf false}(\Imc_i)\cap {\sf false}(\Imc_j) = \emptyset$ (respectively,  
${\sf false}(\Imc_j)\cap {\sf false}(\Imc_k) = \emptyset$) does not hold then  the pair 
$(\Imc_i, \Imc_j)$ (respectively, $(\Imc_j, \Imc_k)$) is a ${\sf goodCandidate}$ 
(Definition~\ref{def:good}), which contradicts Lemma~\ref{lem:nogood}. 
In our proof by contradiction, we show that conditions (i), (ii) and (iii) of Definition~\ref{def:good} hold for 
both $(\Imc_i, \Imc_j)$ and $(\Imc_j, \Imc_k)$. 
\begin{itemize}
\item For condition (i):
 assume to the contrary that 
${\sf true}(\Imc_i)\subseteq {\sf true}(\Imc_j)$. 
By Lemma~\ref{lem:recursive}, we know that ${\sf true}(  \Jmc)\subseteq {\sf true}(\Imc_i)$. 
As 
 ${\sf true}(  \Jmc)\subseteq {\sf true}(\Imc_i)\subseteq {\sf true}(\Imc_j)$ and
$\Imc_j \not\models$ \Call{BuildClauses}{$  \Jmc,\Pmf$},  
by Lemma~\ref{lem:point-i}, we have  
$\Imc_j \not\models$ \Call{BuildClauses}{$ \Imc_i,\Pmf$}, 
which is a contradiction with Lemma~\ref{lem:curious}. 
Now, we assume to the contrary that 
${\sf true}(\Imc_j)\subseteq {\sf true}(\Imc_k)$. 

As   
 ${\sf true}( \Jmc)\subseteq {\sf true}(\Imc_j)\subseteq {\sf true}(\Imc_k)$ and 
$\Imc_k \not\models$ \Call{BuildClauses}{$  \Jmc,\Pmf$},  
by Lemma~\ref{lem:point-i}, we have   
$\Imc_k \not\models$\Call{BuildClauses}{$ \Imc_j,\Pmf$},
which is a contradiction with Lemma~\ref{lem:curious}.
\item For condition (ii):  as $\Imc_j\models \Hmc \setminus \Hmc_j$ 
(Lemma \ref{lem:curious}) we have $\Imc_j\models \Hmc \setminus (\Hmc_i \cup \Hmc_j)$. 
By the same argument $\Imc_i\models \Hmc \setminus (\Hmc_i \cup \Hmc_j)$. 
If ${\sf false}(\Imc_i)\cap {\sf false}(\Imc_j) \neq \emptyset$ then, 
by Lemma \ref{lem:technical}, $\Imc_i \cap \Imc_j \models \Hmc \setminus (\Hmc_i \cup \Hmc_j)$. 
In fact, by the argument above for condition (i), 
${\sf true}(\Imc_i)\not\subseteq {\sf true}(\Imc_j)$. 
So, we actually have $\Imc_i \cap \Imc_j \models \Hmc$. 
Similarly, as $\Imc_j\models \Hmc \setminus \Hmc_j$ (Lemma \ref{lem:curious}) 
we have $\Imc_j\models \Hmc \setminus (\Hmc_j \cup \Hmc_k)$. By the same 
argument $\Imc_k\models \Hmc \setminus (\Hmc_j \cup \Hmc_k)$. 
If ${\sf false}(\Imc_j)\cap {\sf false}(\Imc_k) \neq \emptyset$ then, by 
Lemma \ref{lem:technical} and the fact that 
${\sf true}(\Imc_j)\not\subseteq {\sf true}(\Imc_k)$, we have $\Imc_j \cap \Imc_k\models \Hmc$. 
\item For condition (iii): suppose to the contrary that 
$\Imc_i \cap \Imc_j \models \Tmc$. 

As $\Imc_j \not\models$\Call{BuildClauses}{$   \Jmc,\Pmf$} 
and (by Lemma~\ref{lem:recursive}) ${\sf true}(  \Jmc)\subseteq {\sf true}(\Imc_i)$, 
we have that $\Imc_j \cap \Imc_i \not\models$\Call{BuildClauses}{$   \Jmc,\Pmf$}. 
Then, the condition in Line \ref{line:special} of Algorithm \ref{alg:remove} is satisfied. 
So  
Algorithm \ref{alg:remove} appends $ \Imc_i\cap\Imc_j $ to \Pmf and 
recursively calls Function `UpdatePositiveExamples' 
  with $\Jmc$, $\Pmf$ and 
\Lmf as input. Then,  by Lemma~\ref{lem:technical2}, 
$ \Imc_i\cap\Imc_j \models $\Call{BuildClauses}{$  \Jmc,\Pmf $}. 
Then, in Line~\ref{line:remove-set} of Algorithm~\ref{alg:learner}, 
  $\Imc_j \models$\Call{BuildClauses}{$ \Jmc,\Pmf$}, which is a contradiction. 
Similarly, suppose to the contrary that 
$\Imc_j \cap \Imc_k  \models \Tmc$.
As both $\Imc_k,\Imc_j$ do not satisfy \Call{BuildClauses}{$  \Jmc,\Pmf$}, 
the condition in Line \ref{line:special} of Algorithm \ref{alg:remove} is satisfied. 
So  
Algorithm \ref{alg:remove} appends $ \Imc_j\cap\Imc_k $ to \Pmf and 
recursively calls Function `UpdatePositiveExamples' 
  with $\Jmc$, $\Pmf$ and 
\Lmf as input.
 Then, by   Lemma~\ref{lem:technical2}, 
$ \Imc_j\cap\Imc_k \models$\Call{BuildClauses}{$   \Jmc,\Pmf$}. 
Hence, when Line~\ref{line:remove-set} of Algorithm~\ref{alg:learner} is 
executed both $\Imc_j, \Imc_k $ satisfy \Call{BuildClauses}{$   \Jmc,\Pmf$}, which is a contradiction. 

\end{itemize}
So conditions (i), (ii) and (iii) of Definition~\ref{def:good} 
hold for   $(\Imc_i, \Imc_j)$ and $(\Imc_j, \Imc_k)$, which contradicts Lemma~\ref{lem:nogood}. 
Then,  ${\sf false}(\Imc_i)\cap {\sf false}(\Imc_j) = \emptyset$ and ${\sf false}(\Imc_j)\cap {\sf false}(\Imc_k) = \emptyset$.
\end{proof}

\section{Proof of Theorem~\ref{th:reduction} }

For convenience, we restate our definition of reduction presented in Section~\ref{sec:reduction}. 

\begin{definition}\label{def:reduction2}  
A learning framework $\Fmf=(E, \Lmc, \mu)$
 \emph{polynomial time reduces} to $\Fmf' = (E', \Lmc, \mu')$ if, 
for some  polynomials $p_1(\cdot)$, $p_2(\cdot,\cdot)$ and $p_3(\cdot,\cdot)$ 
there exist 
a function $\f: \Lmc\times E'\to \{\text{ `yes'},\text{ `no'}\}$
and a partial function  
$\g: \Lmc\times\Lmc\times E\to E'$, defined for every $(l,h,e)$ such that
$|h| \leq p_1(|l|)$, 
for which the following conditions hold: 
\begin{itemize} 
\item for all $e'\in E'$ we have $e'\in\mu'(l)$ iff $\f(l,e') = \text{ `yes'}$; 
\item for all $e\in E$ we have $e \in \mu(l)\oplus \mu(h)$ iff 
$\g(l,h,e) \in \mu'(l)\oplus \mu'(h)$;
\item $\f(l,e')$ and $\g(l,h,e)$ are computable in time $p_2(|l|,|e'|)$ and $p_3(|l|,|e|)$, respectively,   
and $l$ can only be accessed by calls to the 
membership oracle ${\sf MEM}_{l,\Fmf}$. 
\end{itemize}
\end{definition}

\smallskip
  
\noindent
{\bf Theorem~\ref{th:reduction} (restated).}
\textit{Let $\Fmf = (E, \Lmc, \mu)$ and $\Fmf' = (E',\Lmc, \mu')$ be  learning frameworks. If there exists a polynomial
time reduction from $\Fmf$ to  $\Fmf'$ and
 $\Fmf'$ is polynomial time exactly learnable then
$\Fmf$ is polynomial time exactly learnable.}
\smallskip
\noindent

\begin{proof}
Let $A'$ be a polynomial time learning algorithm for $(E',\Lmc,\mu')$. 
We construct a learning algorithm $A$ for  $(E,\Lmc, \mu)$,
using internally the learning algorithm $A'$, as follows. 
As learning $(E,\Lmc, \mu)$ polynomial time reduces to
learning $(E',\Lmc, \mu')$, we have that: 
\begin{itemize} 
\item there are functions $\f:\Lmc\times E'\to \{\text{ `yes'},\text{ `no'}\}$ and $\g:
\Lmc\times\Lmc\times E\to E'$ such that $\f$ maps   
$l \in \Lmc$ and `$e'\in E'$' into `yes' or `no' (depending on whether 
$e'\in \mu'(l)$); and $\g$ transforms a  
counterexample `$e \in E$' into a counterexample `$e' \in E'$'. 
\end{itemize}

So, whenever a membership query with $e' \in
E'$ as input is called by $A'$ we compute $\f(l,e')$ by making calls to the ${\sf
MEM}_{l,\Fmf}$ oracle. We return `yes' to $A'$ if $\f(l,e') = \text{ `yes'}$ and `no' otherwise.  
Whenever an equivalence query with $h \in \Lmc$ as input is called by $A'$ we
pass it on to the ${\sf EQ}_{l,\Fmf}$ oracle.  If it returns `yes' then the
learner succeeded. Otherwise the oracle returns `no'  and provides a 
counterexample $e\in E$. 
Then, we compute $e' = g(l,h,e)$ and
return it to $A'$. Notice that computing $\g(l,h,e)$ may also require posing 
additional membership queries  (recall that $l$ can only be accessed via 
queries to the oracle ${\sf MEM}_{l,\Fmf}$). 

By definition of $\f$ and $\g$, all the answers provided to $A'$ are 
consistent with answers the oracles ${\sf MEM}_{l,\Fmf'}$ and ${\sf EQ}_{l,\Fmf'}$
would provide to $A'$. 
Clearly, if algorithm $A$ terminates then it learns $l$.

It remains to 
prove the polynomial time bound for $A$.
Let $p_1(\cdot)$, $p_2(\cdot)$ and $p_3(\cdot,\cdot)$ 
be the polynomials of Definition~\ref{def:reduction2}, that is,
\begin{itemize}
\item $p_1(|l|)$ is the polynomial bound on $|h|$;
\item $p_2(|l|,|e'|)$ is the polynomial time bound for computing $\f(l,e')$;
\item $p_3(|l|,|e|)$ is the polynomial time bound for computing $\g(l,h,e)$.
\end{itemize}

Let $p(\cdot,\cdot)$ be  a polynomial such that in every run of $A'$, the time 
 used by
$A'$ up to each step of computation is bounded by  $p(|l|,|y'|)$,
where $|l|$ is the size of the target $l \in \Lmc$ and $|y'|$ is the size of
the largest counterexample $y'\in E'$ seen by $A'$ up to that
point of computation. 
As $y'$ is the result of transforming with function $\g$  some
counterexample $y\in E$ given by the ${\sf EQ}_{l,\Fmf}$ oracle to algorithm $A$,
its size $|y'|$ is bounded by $p_3(|l|,|y|)$.
Notice that  $y$ is also the largest counterexample seen so far by $A$.
Thus, at every step of computation the time used by 
$A'$ up to that step is
bounded by a polynomial $p'(|l|,|y|) = p(|l|,p_3(|l|, |y|))$. 

For every membership query with $e'\in E'$ asked by $A'$, the size of $e'$  does not exceed
the polynomial time bound   of $A'$ up to that point, that is, $|e'| \leq p'(|l|,|y|)$. 
Then, the time  needed to transform membership queries and answers to equivalence
queries is bounded by $p'_2(|l|,|y|) = p_2(|l|,p'(|l|, |y|))$ and $p_3(|l|,|y|)$, respectively.
All in all, 
 at every step of computation the time  used by $A$ up to that step is bounded by 
$p'(|l|,|y|)\cdot (p'_2(|l|,|y|)+ p_3(|l|,|y|))$, which is
polynomial in $|l|$ and $|y|$, as required.     
\end{proof}

\section{Reductions among Learning Problems}

We now explain the reducibility of the learning problems presented in Figure \ref{fig:reductions}. 
For convenience, in Figure~\ref{fig:reductions2},  we enumerate the reductions 
\footnote{Note that our reduction in Point (1) of Figure~\ref{fig:reductions2}
is non-proper. Though, in this case one can avoid this by translating
 the hypothesis to Horn whenever the algorithm poses an equivalence query (see Remark~\ref{rem:defs}).  }. 
Points (1) and (6) follow from the fact that one can express any Horn formula with a 
polynomial size MVDF (see Remark~\ref{rem:defs} below). 
Point (2) is given in Subsection~\ref{subsec:2-quasi-red}. 
We then have Point (3), where have the $\text{MVD}_\Rmc\rightarrow \text{MVDF}_\Imc$  
direction proved in Subsection~\ref{subsec:mvd-red} (note that this also gives Point~(8)). 
The other direction, $\text{MVDF}_\Imc \rightarrow \text{MVD}_\Rmc$, can be 
proved with similar arguments. 
 Point (4) follows from the fact that CRFMVF is a restriction of MVDF. 
 We show Point~(5) in Subsection~\ref{subsec:horn-red}. Finally, we show Point (7) in 
 Subsection~\ref{subsec:mvdf-quasi-red}. 
  
 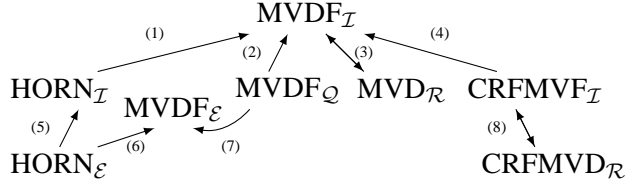
\begin{figure}[h]
\begin{center}
\begin{tikzpicture}
	\begin{pgfonlayer}{nodelayer}
		\node [style=none] (0) at (-5.25, 4) {};
		\node [style=none] (1) at (-3, 3.5) {};
		\node [style=none] (2) at (-8.25, 3.5) {};
		\node [style=none] (3) at (-6, 3.5) {};
		\node [style=none] (4) at (-4.75, 3.5) {};
		\node [style=none] (5) at (-5.5, 4.25) {MVDF$_\Imc$};
		\node [style=none] (6) at (-5.75, 4) {};
		\node [style=none] (7) at (-4.75, 4) {};
		\node [style=none] (8) at (-6.25, 4) {};
		\node [style=none] (9) at (-8.75, 3.25) {HORN$_{\Imc}$};
		\node [style=none] (10) at (-2.5, 3.25) {CRFMVF$_{\Imc}$};
		\node [style=none] (11) at (-8.5, 3) {};
		\node [style=none] (12) at (-2.75, 3) {};
		\node [style=none] (13) at (-8.75, 2.5) {};
		\node [style=none] (14) at (-2.5, 2.5) {};
		\node [style=none] (15) at (-4.25, 3.25) {MVD$_{\Rmc}$};
		\node [style=none] (16) at (-5.75, 3.25) {MVDF$_{\Qmc}$};
		\node [style=none] (17) at (-8.75, 2.25) {HORN$_{\Emc}$};
		\node [style=none] (18) at (-2.25, 2.25) {CRFMVD$_{\Rmc}$};
		\node [style=none] (19) at (-7.25, 3) {MVDF$_\Emc$};
		\node [style=none] (20) at (-6, 4) {};
		\node [style=none] (21) at (-7, 3.25) {};
		\node [style=none] (22) at (-8.25, 2.5) {};
		\node [style=none] (23) at (-7.5, 2.75) {};
		\node [style=none] (24) at (-6.25, 3) {};
		\node [style=none] (25) at (-7, 2.75) {};
		\node [style=none] (26) at (-7, 2.5) {};
		\node [style=none] (27) at (-6, 3) {};
		\node [style=none] (28) at (-7.5, 4) {\tiny{(1)}};
		\node [style=none] (29) at (-3.75, 4) {\tiny{(4)}};
		\node [style=none] (30) at (-6.25, 3.75) {\tiny{ (2)}};
		\node [style=none] (31) at (-4.75, 3.75) {\tiny{(3)}};
		\node [style=none] (32) at (-9, 2.75) {\tiny{(5)}};
		\node [style=none] (33) at (-7.75, 2.5) {\tiny{(6)}};
		\node [style=none] (34) at (-6.5, 2.5) {\tiny{(7)}};
		\node [style=none] (35) at (-3, 2.75) {\tiny{(8)}};
	\end{pgfonlayer}
	\begin{pgfonlayer}{edgelayer}
		\draw [style=newstyle2] (4.center) to (0.center);
		\draw [style=newstyle2] (3.center) to (6.center);
		\draw [style=newstyle2] (2.center) to (8.center);
		\draw [style=newstyle2] (13.center) to (11.center);
		\draw [style=newstyle2] (14.center) to (12.center);
		\draw [style=newstyle2] (12.center) to (14.center);
		\draw [style=newstyle2] (22.center) to (23.center);
		\draw [style=newstyle2, bend left, looseness=1.00] (24.center) to (25.center);
		\draw [style=newstyle2] (0.center) to (4.center);
		\draw [style=newstyle2] (1.center) to (7.center);
	\end{pgfonlayer}
\end{tikzpicture}
\caption{Reductions among learning problems}\label{fig:reductions2}
\end{center}
\end{figure}

We write ${\sf ant}(c)$ (the antecedent) for the set of variables that occur negated in a clause $c$ 
(this set contains $\mathbf{T}$ if no variable occurs negated).

\subsection{Propositional Horn: from Entailments to Interpretations}
\label{subsec:horn-red}

 The learning 
framework $\Fmf\text{(HORN$_\Imc$)}$, studied by 
 \cite{DBLP:journals/ml/AngluinFP92}, is defined as 
 $(E_\Imc,\Lmc_{\sf H},\mu_\Imc)$, where  $\Lmc_{\sf H}$ 
 is the set of all Horn sentences which can be formulated in a set of variables $V$, 
  $E_\Imc$ is the set of interpretations over variables in $V$ and, 
  for a Horn sentence $\Tmc \in \Lmc_{\sf H}$, 
  $\mu_\Imc(\Tmc)$ is defined as $\{\Imc \in E_\Imc \mid \Imc \models \Tmc\}$. 
We also define the  learning 
framework $\Fmf\text{(HORN$_\Emc$)}$, 
studied by \cite{DBLP:conf/icml/FrazierP93}, 
as 
 $(E_\Emc,\Lmc_{\sf H},\mu_\Emc)$, where  $\Lmc_{\sf H}$ 
 is the set of all Horn sentences which can be formulated in a set of variables $V$, 
  $E_\Emc$ is the set of all Horn clauses over variables in $V$ and, 
  for a Horn sentence $\Tmc \in \Lmc_{\sf H}$, 
  $\mu_\Emc(\Tmc)$ is defined as $\{c \in E_\Emc \mid \Tmc \models c\}$.

An algorithm to learn Horn sentences from entailments is presented by 
\cite{DBLP:conf/icml/FrazierP93}, where the authors mention that 
their solution is in fact an application of the learning from interpretations 
algorithm presented by \cite{DBLP:journals/ml/AngluinFP92} with some twists. 
Here we give an alternative proof, based on Theorem \ref{th:reduction}, 
 which shows that 
learning Horn sentences from entailments can be reduced in polynomial time 
to learning Horn sentences from interpretations. 
To give our proof by reduction we use Angluin's \cite{DBLP:journals/ml/AngluinFP92} algorithm as a `black box' and:
(1) transform 
counterexamples given by equivalence queries, which come as entailments into interpretations; and
(2) transform the membership queries, which come as interpretations
into entailments. 
Let \Tmc be the target Horn sentence and \Hmc the learner's hypothesis.    
The following lemma shows how one can simulate an equivalence query by transforming 
a counterexample in the learning from entailments scenario into a 
 counterexample in the learning from interpretations scenario.

\begin{lemma}\label{lem:horn-mem}
Let $\Fmf\text{(HORN$_\Emc$)} = (E_\Emc,\Lmc_{\sf H},\mu_\Emc)$ be the 
learning Horn  from entailments framework and 
$\Fmf\text{(HORN$_\Imc$)} = (E_\Imc,\Lmc_{\sf H},\mu_\Imc)$ 
be the learning Horn from interpretations framework. 
Assume that the target \Tmc and the hypothesis \Hmc are in variables $V$ and 
$|\Hmc|$ is polynomial in $|\Tmc|$.   
If 
$c \in \mu_\Emc(\Tmc) \oplus \mu_\Emc(\Hmc)$ then  
one can construct in  time polynomial in $|\Tmc|$ an interpretation $\Imc$ 
such that $\Imc \in \mu_\Imc(\Tmc) \oplus \mu_\Imc(\Hmc)$. 
\end{lemma}

\begin{proof} 
We show how one can transform a Horn clause $c$ that is a \emph{positive} counterexample (in $\Fmf\text{(HORN$_\Emc$)}$)
into a \emph{negative} counterexample (in $\Fmf\text{(HORN$_\Imc$)}$) and vice-versa. 
If $\Tmc \not\models c$ and $\Hmc \models c$ then 
we construct an interpretation \Imc as the result of initially setting 
${\sf true}(\Imc) = {\sf ant}(c)$ and then exhaustively applying the following rule:
\begin{itemize}
\item if $\Tmc \models \bigwedge_{v \in {\sf true}(\Imc)} \rightarrow w$ (checked with membership query to ${\sf MEM}_{\Tmc,\Fmf\text{(HORN$_\Emc$)}}$), where $w \in V \setminus {\sf true}(\Imc)$, 
then add $w$ to ${\sf true}(\Imc)$.
\end{itemize}
The resulting \Imc is model of $\Tmc$. 
As $\Tmc \not\models c$ we know that the consequent of $c$ is not in ${\sf true}(\Imc)$. 
Then, since ${\sf ant}(c) \subseteq {\sf true}(\Imc)$, we have that \Imc does not satisfy \Hmc. 
That is, $\Imc \in \mu_\Imc(\Tmc) \oplus \mu_\Imc(\Hmc)$. 
Notice that in this case we made $|V|$ membership queries to the 
oracle ${\sf MEM}_{\Tmc,\Fmf\text{(HORN$_\Emc$)}}$. 
When $\Tmc \models c$ and $\Hmc \not\models c$ the argument is similar 
but  we need to check whether 
$\Hmc \models \bigwedge_{v \in {\sf true}(\Imc)} \rightarrow w$, 
where $w \in V \setminus {\sf true}(\Imc)$.  
Since in this case we evaluate the hypothesis, 
no membership query is necessary to produce a negative counterexample.   
\end{proof}
 
To simulate membership queries we transform an interpretation \Imc into polynomially 
many entailment queries which together decide whether \Imc satisfies \Tmc or not. 

\begin{lemma}\label{lem:horn-eq}
Let $\Fmf\text{(HORN$_\Emc$)} = (E_\Emc,\Lmc_{\sf H},\mu_\Emc)$ be the 
learning Horn  from entailments framework and 
$\Fmf\text{(HORN$_\Imc$)} = (E_\Imc,\Lmc_{\sf H},\mu_\Imc)$ 
be the learning Horn from interpretations framework. 
For any interpretation \Imc of a target concept representation $\Tmc\in \Lmc_{\sf H}$, one can decide in polynomial time  in $|\Tmc|$ whether 
$\Imc \in \mu_\Imc(\Tmc)$. 
\end{lemma}

\begin{proof}
A very straightforward algorithm to decide whether \Imc satisfies \Tmc is described as follows. 
Let $C = \{\bigwedge_{v \in {\sf true}(\Imc)} \rightarrow z \mid z \in {\sf false}(\Imc)\}$. For every $c \in C$ the algorithm  
calls ${\sf MEM}_{\Tmc,\Fmf\text{(HORN$_\Emc$)}}$ asking whether $\Tmc \models c$. 
If the answer to any of these queries is `yes' then return `no'. That is, \Imc does not satisfy \Tmc. 
Otherwise, return `yes', \Imc satisfies \Tmc. 
\end{proof}

Lemmas \ref{lem:horn-mem} and \ref{lem:horn-eq} show 
how one can compute, respectively, $\g$   
and $\f$ described in Definition \ref{def:reduction2}. 
Then, using Theorem \ref{th:reduction}, we obtain an alternative proof for  the result   
presented by \cite{DBLP:conf/icml/FrazierP93}. 

\begin{theorem}[\cite{DBLP:conf/icml/FrazierP93}]
The problem of learning propositional Horn from entailments, more precisely, the 
learning   framework $\Fmf\text{(HORN$_\Emc$)}$, is polynomial time exactly learnable. 
\end{theorem}

 \subsection{Multivalued Dependency Formulas: from $2$-quasi-Horn to Interpretations}
\label{subsec:2-quasi-red}

 The learning  
framework $\Fmf\text{(MVDF$_\Qmc$)}$, studied by the authors of 
 \cite{hermo2015exact}, is formally defined as 
 $(E_\Qmc,\Lmc_{\sf M},\mu_\Qmc)$, where  $\Lmc_{\sf M}$ 
 is the set of all MVDFs which can be formulated in a set of variables $V$, 
  $E_\Qmc$ is the set of $2$-quasi-Horn clauses over variables in $V$ and, 
  for a MVDF $\Tmc \in \Lmc_{\sf M}$, 
  $\mu_\Qmc(\Tmc)$ is defined as $\{e \in E_\Qmc \mid \Tmc \models e\}$. 

We show that learning MVDF from $2$-quasi-Horn clauses is reducible to 
learning MVDF from interpretations. 
More precisely, $\Fmf\text{(MVDF$_\Qmc$)}$ polynomial time reduces to 
$\Fmf\text{(MVDF$_\Imc$)}$. 
 To give our proof by reduction we use the algorithm presented in Section~\ref{sec:mvdf-interpretation} 
 as a `black box' and:
(1) transform the membership queries, which come as interpretations
into $2$-quasi-Horn clauses; and 
(2) transform 
counterexamples given by equivalence queries, which come as $2$-quasi-Horn clauses into interpretations. 
Let \Tmc be the target MVDF and \Hmc the learner's hypothesis.  
To simulate membership queries we transform an interpretation \Imc into polynomially 
many $2$-quasi-Horn queries which together decide whether \Imc satisfies \Tmc or not. 

\begin{lemma}\label{lem:mvdf-q-eq}
Let $\Fmf\text{(MVDF$_\Qmc$)} = (E_\Qmc,\Lmc_{\sf M},\mu_\Qmc)$ be the 
learning MVDF  from $2$-quasi-Horn framework and 
$\Fmf\text{(MVDF$_\Imc$)} = (E_\Imc,\Lmc_{\sf M},\mu_\Imc)$ 
be the learning MVDF from interpretations framework. 
For any interpretation \Imc of a target concept 
representation $\Tmc\in \Lmc_{\sf M}$, one can decide in polynomial time  in $|\Tmc|$ whether 
$\Imc \in \mu_\Imc(\Tmc)$. 
\end{lemma}

\begin{proof}
A very straightforward algorithm to decide whether \Imc satisfies \Tmc is described as follows. 
Let $C = \{\bigwedge_{v \in {\sf true}(\Imc)} \rightarrow w\vee z \mid w,z \in {\sf false}(\Imc)\}
\cup \{V\rightarrow \mathbf{F}\mid {\sf true}(\Imc)=V \}$. For every $c \in C$ the algorithm  
calls ${\sf MEM}_{\Tmc,\Fmf\text{(MVDF$_\Qmc$)}}$ asking whether $\Tmc \models c$. 
If the answer to any of these queries is `yes' then return `no'. That is, \Imc does not satisfy \Tmc. 
Otherwise, return `yes', \Imc satisfies \Tmc. 
\end{proof}   
We note that in the learning framework $\Fmf\text{(MVDF$_\Qmc$)}$ 
one can use the membership oracle 
to ensure that at all times $\Tmc\models\Hmc$. Then, we can assume w.l.o.g. that all 
counterexamples given by the oracle are positive.
To transform 
 positive counterexamples, 
 we employ the following result from \cite{hermo2015exact}.

\begin{lemma}(Direct Adaptation from \cite{hermo2015exact})\label{lem:mvdf-q-mem}
Let $\Fmf\text{(MVDF$_\Qmc$)} = (E_\Qmc,\Lmc_{\sf M},\mu_\Qmc)$ be the 
learning MVDF  from $2$-quasi-Horn framework and 
$\Fmf\text{(MVDF$_\Imc$)} = (E_\Imc,\Lmc_{\sf M},\mu_\Imc)$ 
be the learning MVDF from interpretations framework. 
Assume that the target \Tmc and the hypothesis \Hmc are in variables $V$ and 
$|\Hmc|$ is polynomial in $|\Tmc|$.   
If 
$c \in \mu_\Qmc(\Tmc) \oplus \mu_\Qmc(\Hmc)$ is a positive counterexample then  
one can construct in  time polynomial in $|\Tmc|$ an interpretation $\Imc$ 
such that $\Imc \in \mu_\Imc(\Tmc) \oplus \mu_\Imc(\Hmc)$ is a negative counterexample. 
\end{lemma}

The proof of Lemma \ref{lem:mvdf-q-mem} in \cite{hermo2015exact}
involves the construction of a polynomial size 
semantic tree for the hypothesis \Hmc. 
The transformation of negative $2$-quasi-Horn counterexamples  is also possible. In this case, 
we would require additional (polynomially many) membership queries to construct a semantic tree.
Lemmas \ref{lem:mvdf-q-eq} and \ref{lem:mvdf-q-mem} show 
how one can compute, respectively, $\f$   
and $\g$ described in Definition \ref{def:reduction2}. 
Then, using Theorem \ref{th:reduction}, we obtain an alternative proof for  the result   
presented by \cite{hermo2015exact}. 

\begin{theorem}[\cite{hermo2015exact}]
The problem of learning MVDF from $2$-quasi-Horn clauses, more precisely, the 
learning   framework $\Fmf\text{(MVDF$_\Qmc$)}$, is polynomial time exactly learnable. 
\end{theorem}

The  difficulty in showing a reduction in the other direction, 
from $\Fmf\text{(MVDF$_\Imc$)}$ to $\Fmf\text{(MVDF$_\Qmc$)}$, 
is 
to decide whether the target  entails   
a $2$-quasi-Horn clause  using polynomially many 
membership queries with interpretations as input.

\subsection{Multivalued Dependency Formulas: from $2$-quasi-Horn to Entailments (mvd clauses)}
\label{subsec:mvdf-quasi-red}

 The learning  
framework $\Fmf\text{(MVDF$_\Emc$)}$  is   defined as 
 $(E_\Emc,\Lmc_{\sf M},\mu_\Emc)$, where  $\Lmc_{\sf M}$ 
 is the set of all MVDFs which can be formulated in a set of variables $V$, 
  $E_\Emc$ is the set of mvd clauses over variables in $V$ and, 
  for a MVDF $\Tmc \in \Lmc_{\sf M}$, 
  $\mu_\Emc(\Tmc)$ is defined as $\{e \in E_\Emc \mid \Tmc \models e\}$. 

We show that learning MVDF from $2$-quasi-Horn clauses is reducible to 
learning MVDF from entailments. 
More precisely, $\Fmf\text{(MVDF$_\Qmc$)}$ polynomial time reduces to 
$\Fmf\text{(MVDF$_\Emc$)}$. 
 To reduce the problem we:
(1) transform the membership queries, which come as mvd clauses 
into $2$-quasi-Horn clauses; and  
(2) transform 
counterexamples given by equivalence queries, which come as $2$-quasi-Horn clauses into mvd clauses. 
Let \Tmc be the target MVDF and \Hmc the learner's hypothesis.    
The next lemma is immediate, it follows from the fact that any mvd clause 
is equivalent to polynomially many $2$-quasi-Horn clauses 
(see Remark~\ref{rem:defs}). 

\begin{lemma}\label{lem:mvdf-e-eq}
Let $\Fmf\text{(MVDF$_\Qmc$)} = (E_\Qmc,\Lmc_{\sf M},\mu_\Qmc)$ be the 
learning MVDF  from $2$-quasi-Horn framework and 
$\Fmf\text{(MVDF$_\Emc$)} = (E_\Emc,\Lmc_{\sf M},\mu_\Emc)$ 
be the learning MVDF from entailments framework. 
For any mvd clause $c$ of a target concept 
representation $\Tmc\in \Lmc_{\sf M}$, one can decide in polynomial time  in $|\Tmc|$ whether 
$c \in \mu_\Emc(\Tmc)$. 
\end{lemma}


 Lemma~\ref{lem:mvdf-e-mem} shows how one can transform the counterexamples. 
To show Lemma~\ref{lem:mvdf-e-mem}, we use 
the following technical lemma, proved by \cite{hermo2015exact}.

\begin{lemma}[\cite{hermo2015exact}]\label{lem:2-quasi-horn-mvd}
Let \Tmc be a set of mvd clauses formulated in $V$. 
If $\Tmc \models V_1 \rightarrow V_2 \vee V_3$ then 
either $\Tmc \models V_1 \rightarrow (V_2\cup\{v\}) \vee V_3$ or 
$\Tmc \models V_1 \rightarrow V_2 \vee (V_3\cup\{v\})$, where $V_1,V_2,V_3,\{v\} \subseteq V$ and  
 $V_2,V_3$ are non-empty.
 \end{lemma}
 
\begin{lemma}\label{lem:mvdf-e-mem}
Let $\Fmf\text{(MVDF$_\Qmc$)} = (E_\Qmc,\Lmc_{\sf M},\mu_\Qmc)$ be the 
learning MVDF  from $2$-quasi-Horn framework and 
$\Fmf\text{(MVDF$_\Emc$)} = (E_\Emc,\Lmc_{\sf M},\mu_\Emc)$ 
be the learning MVDF from entailments framework. 
Assume that the target \Tmc and the hypothesis \Hmc are in variables $V$ and 
$|\Hmc|$ is polynomial in $|\Tmc|$.   
If   
$c \in \mu_\Qmc(\Tmc) \oplus \mu_\Qmc(\Hmc)$ then  
one can construct in  time polynomial in $|\Tmc|$ an mvd clause $c'$ 
such that $c' \in \mu_\Emc(\Tmc) \oplus \mu_\Emc(\Hmc)$. 
\end{lemma}

\begin{proof}
We show how one can transform a $2$-quasi-Horn clause $X\rightarrow v \vee w$ that is a
  positive  counterexample (in $\Fmf\text{(MVDF$_\Qmc$)}$)
into a  positive  counterexample (in $\Fmf\text{(MVDF$_\Emc$)}$). 
If $\Tmc  \models X\rightarrow v \vee w$ and $\Hmc \not\models X\rightarrow v \vee w$ then 
we construct an mvd clause as the result of initially setting 
$W  = V \setminus (X\cup \{v,w\})$,  $Y = \{v\}$ and $Z = \{w\}$ and 
then   applying the following rule until $X\cup Y \cup Z=V$:
\begin{itemize}
\item if $\Tmc \models X \rightarrow (Y\cup\{w'\}) \vee Z$, where $w' \in W$, (checked by posing 
 membership queries
 to ${\sf MEM}_{\Tmc,\Fmf\text{(MVDF$_\Qmc$)}}$, as in Remark~\ref{rem:defs})   
then add $w'$ to $Y$. Otherwise, add $w'$ to $Z$.
\end{itemize}
By Lemma~\ref{lem:2-quasi-horn-mvd} either $\Tmc \models X \rightarrow (Y\cup\{w'\}) \vee Z$ 
or   $\Tmc \models X \rightarrow Y \vee (Z\cup\{w'\})$ must hold. 
Then, $\Tmc\models X\rightarrow Y \vee Z $. 
As $\{X\rightarrow Y \vee Z\}\models X\rightarrow v \vee w$, we  have
that $\Hmc\not\models X\rightarrow Y \vee Z $. 
That is, $X\rightarrow Y \vee Z \in \mu_\Emc(\Tmc) \oplus \mu_\Emc(\Hmc)$. 
When $\Tmc \not\models X\rightarrow v \vee w$ and $\Hmc \models X\rightarrow v \vee w$ 
the argument is similar 
but  we need to check whether 
$\Hmc \models X \rightarrow (Y\cup\{w'\}) \vee Z$, 
where $w' \in W$.  
Since in this case we evaluate the hypothesis, 
no membership query is necessary to produce a negative counterexample.   
\end{proof}

Lemmas   \ref{lem:mvdf-e-eq} and \ref{lem:mvdf-e-mem}  show 
how one can compute, respectively, $\f$   
and $\g$ described in Definition \ref{def:reduction2}, and, so,
   $\Fmf\text{(MVDF$_\Qmc$)}$ polynomial time reduces to 
$\Fmf\text{(MVDF$_\Emc$)}$. 
The  difficulty in showing a reduction in the other direction, 
from $\Fmf\text{(MVDF$_\Emc$)}$ to $\Fmf\text{(MVDF$_\Qmc$)}$, 
is  
to decide whether the target  entails   
a $2$-quasi-Horn clause  using polynomially many 
membership queries with mvd clauses as input.

\end{document}